\documentclass[a4paper,11pt]{article}
\usepackage[utf8]{inputenc}
\usepackage{amsmath}
\usepackage{amsfonts}
\usepackage{amssymb}
\usepackage{hyperref}
\usepackage{cleveref}
\usepackage{amsthm}
\usepackage{bbm}
\usepackage{xcolor, tcolorbox}
\usepackage{mathtools}
\usepackage{caption}
\usepackage{subcaption}
\usepackage{fullpage} 
\usepackage[toc,page]{appendix}
\usepackage{soul}
\usepackage{times}
\usepackage{natbib}

\DeclarePairedDelimiter\ceil{\lceil}{\rceil}

\newtheorem{theorem}{Theorem}

\newtheorem{definition}{Definition}
\newtheorem{claim}{Claim}

\newtheorem{corollary}{Corollary}

\newtheorem{prop}{Proposition}
\newtheorem{lemma}{Lemma}

\newcommand{\E}{\mathop{\mathbb{E}}}
\def\X{\mathcal{X}}
\def\H{\mathcal{H}}
\def\F{\mathcal{F}}
\def\P{\mathcal{P}}

\def\Y{\mathcal{Y}}
\def\A{\mathcal{A}}

\def\F{\mathcal{F}}

\def\Z{\mathcal{Z}}
\def\k{\mathtt{k}}
\def\L{\mathtt{L}}
\def\M{\mathtt{M}}
\def\R{\mathtt{R}}

\def\T{\mathcal{T}}

\tcbset{
    rounded corners,
    colback = white,
    before skip = 0.2cm,
    after skip = 0.5cm,
    boxrule = 1.5pt,
    arc = 5pt
}

\title{List Online Classification}
\author{Shay Moran\footnote{Departments of Mathematics and Computer Science, Technion and Google Research.
Robert J.\ Shillman Fellow; supported by ISF grant 1225/20, by BSF grant 2018385, by an Azrieli Faculty Fellowship, by Israel PBC-VATAT, by the Technion Center for Machine Learning and Intelligent Systems (MLIS), and by the European Union (ERC, GENERALIZATION, 101039692). Views and opinions expressed are however those of the author(s) only and do not necessarily reflect those of the European Union or the European Research Council Executive Agency. Neither the European Union nor the granting authority can be held responsible for them.} \and
Ohad Sharon \footnote{Department of Mathematics, Technion.} \and
Iska Tsubari\footnote{Department of Mathematics, Technion.
Supported by the European Union (ERC, GENERALIZATION, 101039692). Views and opinions expressed are however those of the author(s) only and do not necessarily reflect those of the European Union or the European Research Council Executive Agency. Neither the European Union nor the granting authority can be held responsible for them.} \and
Sivan Yosebashvili \footnote{Mobileye. Part of this work was done while the author was at the Computer Science Department at the Technion.}}

\begin{document}

\maketitle

\begin{abstract}
  We study multiclass online prediction where the learner can predict using a list of multiple labels (as opposed to just one label in the traditional setting).
  We characterize learnability in this model using the $b$-ary Littlestone dimension.
  This dimension is a variation of the classical Littlestone dimension with the difference that binary mistake trees are replaced with $(k+1)$-ary mistake trees, where $k$ is the number of labels in the list.
  In the agnostic setting, we explore different scenarios depending on whether the comparator class consists of single-labeled or multi-labeled functions and its tradeoff with the size of the lists the algorithm uses. 
  We find that it is possible to achieve negative regret in some cases and provide a complete characterization of when this is possible.
    

  As part of our work, we adapt classical algorithms such as Littlestone's SOA and Rosenblatt's Perceptron to predict using lists of labels.
  We also establish combinatorial results for list-learnable classes, including a list online version of the Sauer-Shelah-Perles Lemma.
  We state our results within the framework of pattern classes --- a generalization of hypothesis classes which can represent adaptive hypotheses 
  (i.e.\ functions with memory), and model data-dependent assumptions such as linear classification with margin.%

\end{abstract}

\section{Introduction}

In certain situations where supervised learning is used, it is acceptable if the prediction is presented as a short list of candidate outputs. 
For instance, recommendation systems like those used by Netflix, Amazon, and YouTube (see Figure~\ref{fig:youtube}) recommend a list of movies or products to their users for selection (as opposed to just a single movie/product). 
In fact, even in tasks where the objective is to predict a single label, it may be helpful to first reduce the output space (which can potentially be very large) by learning a short list of options.
For instance, medical doctors could consult an algorithm that generates a short list of potential diagnoses based on a patient's medical data.



List prediction rules naturally arise in the setting of conformal learning.
In this model, algorithms make their predictions while also offering some indication of the level of uncertainty or confidence in those predictions.
For example in multiclass classification tasks, given an unlabeled test example~$x$, the conformal learner might output a list of all possible classes along with scores which reflect the probability that $x$ belongs to each class. 
This list can then be truncated to a shorter one which contains only the classes with the highest score. 
We refer the reader to the book by~\citet*{VGS05} and the surveys by~\citet*{ShaferV08,AB21} for more details.

\begin{figure}
\centering
  \centering
  \includegraphics[width=.75\linewidth]{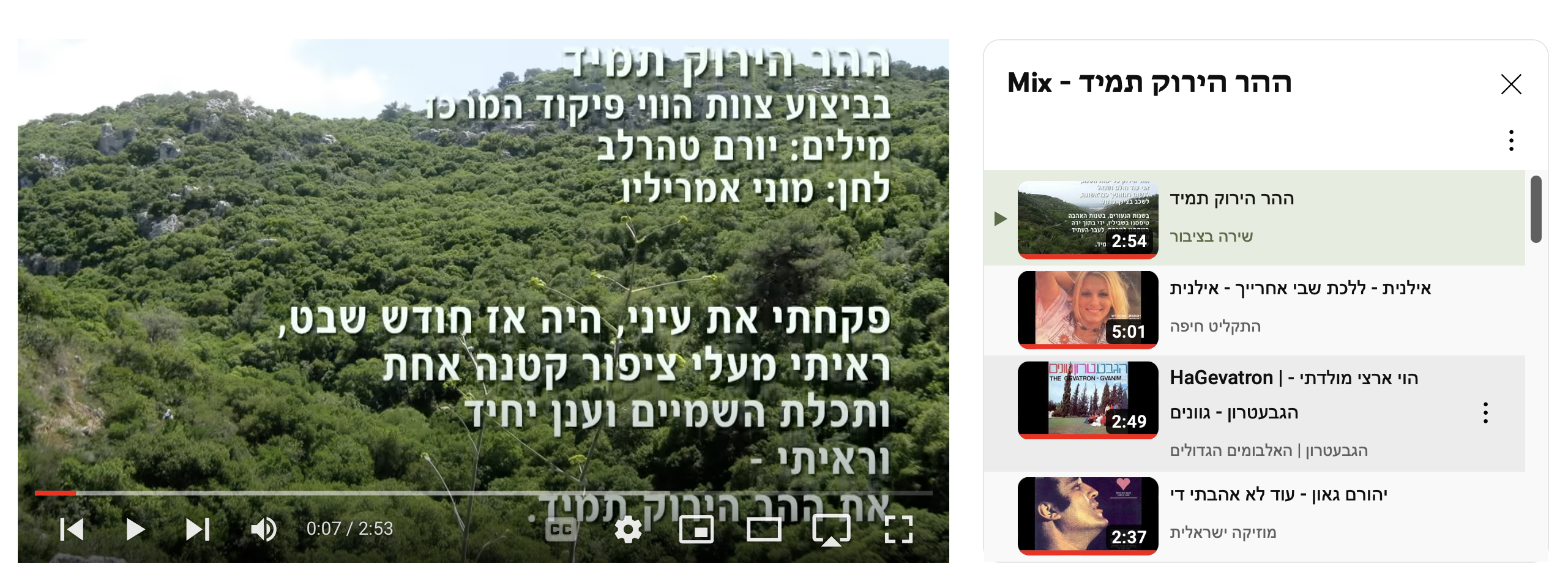}
  \caption{YouTube's recommendation system generates a short list of videos to users based on their past watch history and likes/dislikes.}
  \label{fig:youtube}
\end{figure}

Recently, list learning has been examined within the PAC (Probably Approximately Correct) framework. \citet*{Brukhim22Multiclass} proved a characterization of multiclass PAC learning using the Daniely-Shwartz (DS) dimension. 
A central idea in this work is the use of list learners to decrease the (possibly infinite) label space; this enables them to learn the reduced problem using known methods which apply when the number of labels is bounded. 
In a subsequent study,~\citet*{CharikarP22} investigated which classes can be learned in the list PAC model, and gave a characterization of list learnability using a natural extension of the DS dimension. In the process, they identified key characteristics of PAC learnable classes that extend naturally (but not trivially) to the list setting.
    


\smallskip

The purpose of this study is to examine list learnability in the setting of multiclass online classification. 
Specifically, we aim to understand:

\begin{tcolorbox}
\begin{center}
Guiding Questions
\end{center}

\begin{center}
\begin{itemize}
\item 
What tasks are online learnable by algorithms that make predictions using short lists of labels?
Is there a natural dimension that measures list online learnability?
\item How does the possibility of using lists impact the mistake- and regret-bounds?
How do these bounds improve as the algorithm is allowed to use larger lists?

\end{itemize}
\end{center}
\end{tcolorbox}
We study these questions in the context of multiclass online prediction using the 0/1 classification loss. 
Extensions to weighted and continuous losses are left for future research.

\subsection{Our Contribution}

\paragraph{Realizable Case (Theorems~\ref{thm:realchar} and~\ref{thm:realquant}).}
Our first main result characterizes list learnability using a natural variation of the Littlestone dimension~\citep{Littlestone88online}.
We show that a class is $k$-list online learnable (i.e.\ can be learned by an algorithm that predicts using list of size $k$)
if and only if its $(k+1)$-ary Littlestone dimension is bounded. 
The latter is defined as the maximum possible depth of a complete $(k+1)$-ary mistake tree that is shattered by the class.
We further show that the list Littlestone dimension is equal to the optimal mistake bound attainable by deterministic learners.

\paragraph{Agnostic Case (Theorems~\ref{thm:agnchar} and~\ref{thm:negreg}).}
We characterize agnostic list online learnability by showing it is also captured by a bounded list Littlestone dimension
(and hence agnostic and realizable list learning are equivalent).
Our proof hinges on an effective transformation that converts a list learner in the realizable case to an agnostic list learner. 

We further investigate how does the expressiveness of the comparator class and the resources of the algorithm
impact the optimal regret rates. For example, 
consider the task of agnostic learning a \underline{multi-labeled} comparator class 
consisting of functions which maps an input instance~$x$ to a set of $k'=3$ output labels. 
How does the optimal regret changes as we use list learners with different list size $k=1,2,3,\ldots$?
Clearly , the regret cannot increase as we increase $k$; does the regret decreases? by how much? can it become negative?
We show that once the algorithm can use larger lists than the minimum required for learning, 
then it can achieve negative regret (which can be called \emph{pride}). 
In Theorem~\ref{thm:negreg} we characterize the cases in which negative regret is possible
as a function of the trade-off between the algorithm's resources and the class complexity.


\paragraph{Variations of Classical Results and Algorithms.}
We develop and analyze variations of classical algorithms such as Rosenblatt's Perceptron (Figure~\ref{fig:perceptron}) 
and Littlestone's SOA (Figure~\ref{fig:ListSOA}). The new modified algorithms predict using lists and hence are more expressive.
For example, our $k$-list Perceptron algorithm learns any linearly separable sequence with margin between the top class and the $(k+1)$'th top class (as opposed to margin between the top and the second top classes in the standard multiclass Perceptron).

Our transformation of realizable list learners to agnostic list learners extends the one by \citet*{Ben-david09agnostic}
in the binary case. Along the way we prove an online variation of the Sauer-Shelah-Perles Lemma for list learnable classes (Proposition~\ref{prop:ssp}).

\paragraph{Pattern Classes.} We present our results using a general framework of pattern classes.
In a nutshell, a pattern class $\P$ is a set of sequences of examples; these sequences are thought of as the set of allowable sequences that are expected as inputs. Thus, in the realizable case, the goal is to design a learning rule that makes a bounded number of mistake on every pattern in the class.
Every hypothesis class $\H$ induces a pattern class consisting of all sequences that are realizable by $\H$:
\[\P(\H)=\{S : S\text{ is realizable by }\H\}\] 
However, pattern classes have the advantage of naturally expressing data-dependent assumption such as linear classification with margin (see Section~\ref{sec:perceptron}). In fact, pattern classes are more expressive than hypothesis classes:
in Appendix~\ref{app:patternexample} we give examples of an online learnable pattern class $\P$ such that for every hypothesis class $\H$,
if $\P\subseteq \P(\H)$ then $\H$ is not online learnable. 

Pattern classes are even more expressive than partial concept classes \citep*{LongPartial,AlonHHM21}.
In fact, partial concept classes essentially correspond to symmetric pattern classes:
i.e.\ pattern classes with the property that a permutation of every pattern in the class is also in the class.
Asymmetric pattern classes can be used to represent online functions (or functions with memory).\footnote{Online functions are maps $h:\X^\star \to \Y$ that get as an input a finite sequence of instances $x_1,\ldots, x_T$ and produce a sequence of labels $y_1,\ldots, y_T$, where $y_t=h(x_t; x_{t-1},\ldots, x_1)$ is interpreted as the label of $x_t$. In contrast with oblivious (memoryless) functions, 
the rule that assigns $y_t$ to $x_t$ can depend on the history (i.e.\ on $x_1,\ldots x_{t-1}$). As an example, consider the online function which gets as an input a sequence of points $x_1,\ldots,x_T\in\mathbb{R}^d$ and produces labels $y_1,\ldots y_T\in\{0,1\}$ such that $y_1=0,y_2=1$ and for $t>2$, $y_t$ is set to be the label of the closest point to $x_t$ among $x_1,\ldots x_{t-1}$ (breaking ties arbitrarily). 
One can model online function classes using pattern classes by taking the set of all patterns $(x_1,y_1),\ldots,(x_t,y_t)$
that are obtained by feeding $x_1,\ldots,x_t$ to one of the online functions in the class.}



\subsection{Organization}
We begin with giving the necessary background and formal definitions in Section~\ref{sec:def}.
While this section contains standard and basic definitions in online learning, we still recommend that even the experienced readers will skim through it, especially through the less standard parts which define learnability of pattern classes and introduce the list learning setting.

In Section~\ref{sec:results} we state and present our main results, and in Section~\ref{sec:proofs} we provide the proofs along with some additional results with more elaborate bounds. Finally, Section~\ref{sec:future} contains some suggestions for future research.

\section{Definitions and Background}\label{sec:def}

Let $\X$ be a set called the domain and let $\Y$ be a set called the label space.
We assume that $\Y$ is finite (but possibly very large).
A pair $z=(x,y)\in\X\times \Y$ is called an example,
and $\Z=\X\times \Y$ denotes the space of examples.
An hypothesis (or a concept) is a mapping $h:\X\to\Y$, 
an hypothesis class (or a concept class) $\H$ is a set of hypotheses.
More generally, for $\lvert \Y\rvert \geq  k \geq 1$, a $k$ multi-labeled hypothesis is a map $h:\X\to \binom{\Y}{k}$ taking
an input $x\in \X$ to a set of $k$ labels from $\Y$.
A $k$ multi-labeled hypothesis class is a class of $k$ multi-labeled hypotheses.

\paragraph{Pattern Classes.}
We use pattern classes -- a generalization of hypothesis classes which enables modelling data-dependent assumptions such as linear classification with margin in a unified and natural way.

A pattern $S=\{z_t\}_{t=1}^T$ is a finite sequence of examples.
For $0 < i\leq T$ we let $S_{< i}$ denote the prefix of $S$ which consists of the first $i-1$ examples in~$S$.
For a pair of patterns $S_1,S_2$ we denote by $S_1\circ S_2$ the pattern obtained by concatenating $S_2$ after $S_1$.
We denote by $S_1 \subseteq S_2$ the relation ``$S_1$ is a subsequence of $S_2$''.
For a set $A$, we denote by $A^\star = \cup_{T=0}^\infty A^T$ the set of all finite sequences\footnote{Here, $A^0$ is the set $A^{0}=\{\emptyset\}$ which contains only the empty sequence $\emptyset$.} of elements in $A$. 
Thus,~$\Z^\star$ is the set of all patterns.

A pattern class $\P\subseteq \Z^\star$ is a set of patterns which is \emph{downward closed}: 
for every $S\in \P$, if $S'\subseteq S$ is a subsequence of $S$ then $S'\in \P$.
For a pattern class $\P$, $y \in \Y$ and $x\in \X$, we denote by $\P_{x\rightarrow y}$ the subset of $\P$ consists of all patterns such that their concatenation with the example $(x,y)$ is also in $\P$:
\[\P_{x\rightarrow y} = \{S: (x,y)\circ S\in \P\}.\]
We say that a sequence/pattern $S$ is realizable by~$\P$ if $S\in \P$.
we say that a sequence/pattern $S\in \Z^\star$ is consistent with (or realizable by) a multi-labeled hypothesis $h:\X\to\Y$ if for every $(x_t,y_t)\in S$ we have $y_t\in h(x_t)$ (or $y_t=h(x_t)$ when $h$ is a single-labeled hypothesis).

Notice that every (multi-labeled) hypothesis class $\H$ induces a pattern class:
\[\P(\H)=\{S\in\Z^\star : S \text{ is consistent with some } h\in \H\}.\]
So, a sequence $S$ is realizable by $\H$ if and only if $S\in\P(\H)$.
The pattern class $\P(\H)$ is equivalent to $\H$ from a learning theoretic perspective
in the sense that a learning rule learns $\H$ if and only if it learns $\P(\H)$.
Hence, pattern classes are at least as expressive as hypothesis classes.
The upside is that pattern classes allow to naturally express data-dependent assumptions such as linear separability with margin
(as we study in more detail below).
This is done by simply considering the pattern class which consists of all sequences that satisfy the assumed data-dependent condition.
Pattern classes extend the notion of partial concept classes~\citep{auer1999structural,AlonHHM21}.
In Appendix~\ref{app:patternexample} we give examples of a learnable task that can be represented by a pattern class,
but not by an hypothesis class. That is, a learnable pattern class $\P$ such that every hypothesis class $\H$
for which $\P\subseteq \P(\H)$ is not learnable.

\paragraph{Deterministic Learning Rules.}
A (deterministic) online learner is a mapping $\A: \Z^\star \times \X \to \Y$.
That is, it is a mapping that maps a finite sequence $S\in \Z^\star$ (the past examples),  
and an unlabeled example $x$ (the current test point) to a label $y$, which is denoted by $y=\A(x ; S)$.

Let $k\geq 1$; a deterministic $k$-list learner is a mapping $\A: \Z^\star \times \X \to \binom{\Y}{k}$.
That is, rather than outputting a single label, a $k$-list learner outputs a set of $k$ labels.
For a sequence $S=\{(x_t,y_t)\}_{t=1}^T\in \Z^\star$ we denote by $\M(\A; S)$ the number of mistakes $\A$ makes on $S$:
\[\M\bigl(\A;S\bigr) = \sum_{t=1}^n 1\bigl[y_t\notin \A(x_t ; S_{<t})\bigr].\]

\paragraph{Randomized Learning Rules.}
In the randomized setting we follow the standard convention in the online learning literature
and model randomized learners as deterministic mappings $\A:\Z^\star \times \X \to \Delta(\Y)$,
where $\Delta(\Y)$ denote the space of all distributions over $\Y$.
Thus, $\hat p = \A(x;S)$ is a $\lvert\Y\rvert$-dimensional probability vector representing
the probability of the output label $\A$ assigns to $x$ given $S$ (see e.g.~\citet{cesa2006prediction,Shalev-Shwartz12survey,hazan2019introduction}).
For a sequence $S=\{(x_t,y_t)\}_{t=1}^T\in \Z^\star$ we denote by $\M(\A; S)$ the expected number of mistakes $\A$ makes on $S$:
\begin{align*}
\M\bigl(\A;S\bigr) = \E_{\hat y_t \sim p_t}\Bigl[\sum_{t=1}^T 1[\hat y_t\neq y_t]\Bigr]=  \sum_{t=1}^T\Pr_{\hat y_t \sim \hat p_t}[\hat y_t \neq y_t]. \tag{where $\hat p_t = \A(x_t;S_{<t})$}
\end{align*}
 Similarly, for $k\geq 1$ a randomized $k$-list learner is a mapping $\A:\Z^\star \times \X \to \Delta\bigl(\binom{\Y}{k}\bigr)$, where $\Delta\bigl(\binom{\Y}{k}\bigr)$ is the space over all distributions over subsets of size $k$ of $\Y$.
Thus, the expected number mistakes $\A$ makes on a sequence $S$ is given by
\begin{align*}
\M\bigl(\A;S\bigr) = \sum_{t=1}^T\Pr_{\hat L_t \sim \hat q_t}[y_t\notin \hat L_t]. \tag{where $\hat q_t = \A(x_t;S_{<t})$}
\end{align*}




\paragraph{Realizable Case Learnability (Pattern Classes).}
Let $\P$ be a pattern class and let $k\geq 1$. 
We say that $\P$ is $k$-list online learnable if there exists a learning rule $\A$ and $M\in\mathbb{N}$
such that $\M(\A;S)\leq M$ for every input sequence $S$ which is realizable by $\P$.
An hypothesis class $\H$ is $k$-list online learnable if $\P(\H)$ is $k$-list online learnable.
Let $\M_k(\P)$ denote the optimal mistake bound achievable in learning $\P$ by deterministic learners, 
and let $\R_k(\P)$ denote the optimal expected mistake bound in learning $\P$ by randomized learners.

\paragraph{Agnostic Case Learnability (Hypothesis Classes).}
Pattern classes are suitable to define a learning task in terms of its possible input sequences (patterns).
While it is possible to also define agnostic online learning with respect to pattern classes,
we find it more natural to stick to the traditional approach and define it with respect to hypothesis classes.
For example, hypothesis classes allows to explicitly study the tradeoff between the list-size used by the learning rule,
and the list size used by the functions in the class.

Let $h$ be a $k$ multi-labeled hypothesis. For a sequence $S=\{(x_t,y_t)\}_{t=1}^T$ of examples, 
we let $\M(h;S)=\sum_{t=1}^T 1[y_t\notin h(x_t)]$ denote the number of mistakes $h$ makes on $S$.
Let $\H$ be a multi-labeled hypothesis class and let $\A$ be a $k'$ list learning rule (notice that $k'$ might be different than $k$).
For a sequence $S=\{(x_t,y_t)\}_{t=1}^T$ of examples, denote the regret of $\A$ with respect to $\H$ on $S$ by
\[\R_{\H}(\A;S) = \M(\A;S) - \min_{h\in \H}\M(h;S)\]

For an integer $T$, we let $\R_\H(\A;T)=\max\{\R_{\H}(\A;S) : \lvert S\rvert = T\}$.
We say that $\H$ is $k'$-list agnostic learnable if there exists a learning rule $\A$ and a sublinear function $R(T)=o(T)$ such that  $\R_{\H}(\A;T)\leq R(T)$ for every $T$.

\paragraph{Mistake Trees and Littlestone Dimension.} 
A $b$-ary mistake tree is a decision tree in which each internal node has out-degree $b$.
Each internal node $v$ in the tree is associated with an unlabeled example $x(v)\in \X$ and each edge $e$ is associated with a label $y(e)\in \Y$ such that for every internal node $v$ all its $b$ out-going edges are associated with different labels from $\Y$. A branch is a root-to-leaf path in the tree. 
Notice that every branch 
\[v_0\to_{e_0} v_1\to_{e_1}\ldots \to_{e_{r}}v_{r+1}\]
naturally induces a sequence of examples $\{(x_i,y_i)\}_{i=0}^r$ where $x_i=x(v_i)$ and $y_i=y(v_i\to v_{i+1})$.

We say that a mistake tree $\T$ is shattered by a pattern class $\P$ if every branch in $\T$ induces a sequence of examples in $\P$. A mistake tree $\T$ is shattered by a (possibly multi-labeled) hypothesis class $\H$ if it is shattered by $\P(\H)$.
The $b$-ary Littlestone dimension of a pattern class $\P$, denoted $\L_{b-1}(\P)$, 
is the largest possible depth of a complete $b$-ary mistake tree that is shattered by $\P$.
(Thus, $\L_1(\cdot)$ is the classical Littlestone dimension.)
If $\P$ shatters such trees of arbitrarily large depths then we define $\L_{b-1}(\P)=\infty$.
The $b$-ary Littlestone dimension of a (multi-labeled) hypothesis class $\H$ is $\L_{b-1}(\H):=\L_{b-1}(\P(\H))$.

\section{Results}\label{sec:results}

\subsection{Warmup: List Perceptron}\label{sec:perceptron}
As a warmup, we begin with a natural adaptation of the classical Perceptron algorithm by \citet{rosenblatt1958perceptron}
to the list online learning setting. 
Our list Perceptron algorithm may be seen as a natural adaptation of the multiclass Perceptron algorithm by Kesler~\citep{DudaHart1973},
and the more recent variants by~\citet{CramerS03} and by~\citet*{BeygelzimerPSTW19}.

\paragraph{Mutliclass Linear Separability.}
Let $\X=\mathbb{R}^d$ and let $\Y$ denote the label space. 
A sequence of examples $S=\{(x_t,y_t)\}_{t=1}^T$ is called linearly separable
if for every label $y\in \Y$ there exists a $w_y^\star\in \mathbb{R}^d$ such that for every $t=1,\ldots, T$:
\begin{equation}\label{eq:linsep}
(\forall y\neq y_t): w_{y_t}^\star \cdot x_t > w_y^\star \cdot x_t,
\end{equation}
where ``$\cdot$'' denotes the standard inner product 
in $\mathbb{R}^d$. In words, the label $y$ of each point $x$ corresponds to the direction $w_y^\star$ on which $x$ has the largest projection. We follow standard convention and assume that the $w^\star_y$'s are normalized such that
\begin{equation}\label{eq:norm}
\sum_{y\in \Y}\|w^\star_y\|^2 = 1,
\end{equation}
where $\|\cdot\|$ denotes the $\ell_2$ euclidean norm.
If a sequence of vectors $w_y^\star$ for $y\in\Y$ satisfies Equations~\ref{eq:linsep} and~\ref{eq:norm} above then we say that it
\emph{linearly separates $S$}.

\paragraph{Multiclass Margin.}
For $k>0$, the $k$'th margin of a separable sequence $S=\{(x_t,y_t)\}_{t=1}^T$, denoted by $\gamma_k(S)$ is the maximal real number $\gamma_k>0$ 
for which there exist vectors $w_y^\star$ for $y\in \Y$ which linearly separate $S$ such that for every $t\leq T$, 
\[w_{y_t}^\star\cdot x_t - w_y^\star \cdot x_t \geq \gamma_k, \text{ for all but at most $k$ labels $y\in \Y$.}\]
So, if we sort the labels $y$'s according to the inner product $w_y^\star\cdot x_t$ then $y_t$ is the largest and the $(k+1)$'th largest label $y$ in the sorted list satisfies $(w_{y_t}^\star - w_y^\star)\cdot x_t \geq \gamma_k$.

Notice that $\gamma_1\leq \gamma_2\leq\ldots$ and that $\gamma_1$ is the standard margin in multiclass linear classification.

\begin{figure}
\begin{tcolorbox}
\begin{center}
{\bf List Perceptron Algorithm}
\end{center}
\textbf{Parameters:} $\X=\mathbb{R}^d$, $\Y=$ label space, and $k=$ list size.

\textbf{Input:} A linearly separable sequence $S$ of length $T$.

\textbf{Initialize:} For all $y\in \Y$ set $w_{y}= \vec 0$.

For $t=1,\ldots,T$
\begin{enumerate}
    \item Receive unlabeled example $x_t\in \X$.
    \item Sort the labels $y\in \Y$ in a non-increasing order according to the inner product $w_y \cdot x_t$.
    \item Predict the list $P_t$ which consists of the top $k$ labels in the above order.    
    \item Receive correct label $y_t\in \Y$.
    \item If $y_t \notin P_t$ then update the $w_y$'s as follows:
    \begin{itemize}
        \item $w_{y_t} \leftarrow w_{y_t} + k x_t$,
        \item $w_{y} \leftarrow w_{y} -x_t$ for all $y\in P_t$.
    \end{itemize}
\end{enumerate}
\end{tcolorbox}
\caption{List Perceptron Algorithm}
\label{fig:perceptron}
\end{figure}

\begin{theorem}[List Perceptron Mistake Bound]\label{thm:listperc}
Let $S=\{(x_t,y_t)\}_{t=1}^{T}$ be a linearly separable input sequence and let $R = \max_{t\leq T}\| x_{t}\|$. 
Then, the $k$-list Perceptron algorithm (Figure~\ref{fig:perceptron}) makes at most $k(k+1)\frac{R^2}{\gamma_k^2}$ mistakes on $S$, where $\gamma_k=\gamma_k(S)$ is the $k$'th margin of $S$.
\end{theorem}
We prove Theorem~\ref{thm:listperc} in Section~\ref{sec:percproof}.
The proof follows an adaptation to the list setting of the classical analysis of the Perceptron mistake bound~\citep{rosenblatt1958perceptron,CramerS03,BeygelzimerPSTW19}.

Note that we can also state Theorem~\ref{thm:listperc} in the language of pattern classes.
For $R,\gamma > 0$ let $\P=\P(R,\gamma_k)$ denote the set of all patterns $S=\{(x_t,y_t)\}_{t=1}^T$
whose $k$'th margin is at least $\gamma_k$ and such that $\|x_t\|\leq R$ for all $t\leq T$.
Thus, Theorem~\ref{thm:listperc} implies that $\P$ is online learnable with mistake bound $\leq k(k+1)\frac{R^2}{\gamma_k^2}$.

\subsection{Realizable Case}

We now turn to characterize list online learnability in the realizable setting.
We will state our results in the general framework of pattern classes which generalizes the traditional framework of hypothesis classes, as detailed above.

\subsubsection{Qualitative Characterization}

\begin{theorem}[Pattern Classes]\label{thm:realchar}
Let~$\P$ be a pattern class, and let $k \in \mathbb{N}$ denote the list-size parameter. 
Then, the following statements are equivalent:
\begin{enumerate}
    \item $\P$ is $k$-list online learnable in the realizable setting.
    \item The $(k+1)$-ary Littlestone dimension of $\P$ is finite: $\L_k(\P)<\infty$.
\end{enumerate}
\end{theorem}

As a corollary, we deduce the same characterization for (multi-labeled) hypothesis classes.

\newpage

\begin{corollary}[Hypothesis Classes]
    Let $\H$ be a $k'$ multi-labeled hypothesis class, and let $k \in \mathbb{N}$ denote the list-size parameter. 
Then, the following statements are equivalent:
\begin{enumerate}
    \item $\H$ is $k$-list online learnable in the realizable setting.
    \item The $(k+1)$-ary Littlestone dimension of $\H$ is finite: $\L_k(\H)<\infty$.
\end{enumerate}
\end{corollary}
\begin{proof}
Notice that $\H$ is $k$-list online learnable if and only if $\P(\H)$ is,
and that $\L_k(\H)=\L_k(\P(\H))$. Thus, the corollary follows from Theorem~\ref{thm:realchar}.
\end{proof}

\subsubsection{Quantitative Mistake Bounds}

\begin{theorem}[Optimal Mistake Bound]\label{thm:realquant}
Let $\P$ be a pattern class and recall that $\mathtt{M}_k(\P)$ is the optimal (deterministic) mistake bound 
in $k$-list learning $\P$ in the realizable case. Then,
\[\mathtt{M}_k(\P)= \L_k(\P).\]
Further, the optimal expected mistake bound for randomized algorithms, $\R_k(\P)$, satisfies
\[\M_k(\P)\geq \R_k(\P) \geq \frac{1}{k+1}\M_k(\P),\]
and both inequalities can be tight.
\end{theorem}

As before, this theorem applies verbatim to (multi-labeled) hypothesis classes.

The lower bound in~Theorem~\ref{thm:realquant} follows directly from the definition of mistake trees and $\L_k(\P)$.
The upper bound relies on an adaptation of Littlestone's Standard Optimal Algorithm (SOA)
to the list learning setting and to pattern class (see Figure~\ref{fig:ListSOA}). 

\begin{figure}
\begin{tcolorbox}
\begin{center}
{\bf List SOA}
\end{center}

\textbf{Parameter:} $k=$ list size.

\textbf{Input:} A pattern class $\P$ with domain $\X$ and label space $\Y$.

\textbf{Initialize:} $V=\P$.

For $t=1,2,\ldots$
\begin{enumerate}
    \item Receive unlabeled example $x_t\in \X$.
    \item Sort the labels $y\in \Y$ in a non-increasing order according to the values $\L_k(V_{x_t\rightarrow y})$.
    \item Predict the list $P_t$ which consists of the top $k$ labels in the above order.
    \item Receive correct label $y_t\in \Y$.
    \item If $y_t\notin P_t$ then update $V \leftarrow V_{x_t\rightarrow y_t}$\footnote{One may wonder, why $V$ is not updated at each time step. The reason is because in the proof of the list version of the SSP lemma (See Section \ref{sec:sspproof}), we use the List SOA to construct an online cover, and this construction exploits the property that $V$ is changed only upon making a mistake.}.
\end{enumerate}
\end{tcolorbox}
\caption{List Standard Optimal Algorithm (SOA)}
\label{fig:ListSOA}
\end{figure}

\subsection{Agnostic Case}

We next turn to the agnostic case where we lift the assumption that the input sequence is realizable by the class.
Consequently, the goal of the learner is also adapted to achieve a mistake bound which is competitive with the best function in the class. (See Section~\ref{sec:def} for the formal definition.)
Thus, in the agnostic setting the class is viewed as a \emph{comparator} class with respect to which the algorithm should achieve vanishing regret.

In this section we address the following questions: can every (multi-labeled) hypothesis class~$\H$ that is learnable in the realizable setting be learned in the agnostic setting?
How do the regret bounds behave as a function of the list size used by the learner?
E.g.\ say $\H$ is $k=2$-list learnable and that the learner $\A$ uses lists of size $3$.
Is it possible to achieve negative regret in this case?

Notice that we model the agnostic setting using hypothesis classes and not pattern classes.
We made this choice because we find it more natural to address the above questions, specifically regarding the way the regret depends on the list size of the algorithm and the number of labels each hypothesis $h\in\H$ assigns to a given point.

\subsubsection{Characterization}
We show a similar qualitative result for the agnostic case. That is, a class is agnostic $k$-list online learnable if and only if its $(k+1)$-ary Littlestone dimension is finite.

\begin{theorem}\label{thm:agnchar}
Let $\H$ be a (possibly multi-labeled) hypothesis class, and let $k \in \mathbb{N}$. 
Then, the following statements are equivalent:
\begin{enumerate}
    \item $\H$ is $k$-list online learnable in the agnostic setting.
    \item The $(k+1)$-ary Littlestone dimension of $\H$ in finite, $\L_k(\H)<\infty$.
\end{enumerate}
\end{theorem}
Quantitatively we get the following upper bound on the regret
\begin{equation}
O\Biggl(\sqrt{dT\ln\Bigl(\frac{T}{d}\Bigl\lceil{\frac{L-k}{k}}\Bigr\rceil k\Bigr)}\Biggr),
\end{equation}
where $d=\L_k(\H)$, and $L=\lvert \Y\rvert$.

The proof of Theorem~\ref{thm:agnchar} follows an adaptation to the list setting of the agnostic-to-realizable reduction by~\citet*{Ben-david09agnostic,Daniely15Multiclass}, which hinges on the List Online Sauer-Shelah-Perles Lemma.
We note that like in classical online learning, agnostic list online learners must be randomized.

\subsubsection{Pride and Regret}

Let $\H$ be an hypothesis class such that $\L_k(\H) < \infty$. 
Theorem~\ref{thm:agnchar} implies that there is a $k$-list learner that agnostically learns $\H$ with vanishing regret.
How does the regret changes if we use more resourceful list learners, with list size larger than $k$? Can it become negative? 
The following theorem answers this question.
In a nutshell, it shows that once the algorithm has more resources than necessary for learning $\H$ in the realizable setting,
then it can achieve negative regret.

    \begin{theorem}\label{thm:negreg}
Let $\H$ be a multi-labeled hypothesis class, and assume that $\lvert \H\rvert > 1$. Let $1\leq \k_\H\leq \lvert \Y\rvert$ be the minimal number such that $\L_{\k_\H}(\H)<\infty$. 
For a list-learning rule $\A$ let $\k_\A\leq \lvert \Y\rvert$ denote the size of the lists it uses. Then, if $\A$ agnostically learns $\H$ then necessarily $\k_\A \geq \k_\H$ and the following holds in these cases:
\begin{enumerate}
    \item {\bf Negative Regret.}\label{item:agn1}
    Let $\lvert \Y\rvert \geq k > \k_\H$, then there exists a learning rule $\A$ with $\k_\A=k$ which learns $\H$ with a negative regret in the following sense: its expected regret on every input sequence of length $T$ is at most
    \[(p-1)\cdot \mathtt{opt_\H} + c\cdot \bigl(\sqrt{{T\ln T}}\bigr),\]
    where $p=p(\H)<1$, $c=c(\H)$ both do not depend on $T$, and $\mathtt{opt_\H}$ is the number of mistakes made by the best function in $\H$. (Note that when $\mathtt{opt_\H} = \Omega(T)$ then the regret is indeed negative.)
    \item {\bf Nonnegative Regret.}\label{item:agn2}\footnote{We comment that in the realizable case every learning algorithm trivially has nonnegative regret. Thus, notice that this item states a more elaborate statement which demonstrates a non-trivial sense in which the regret is non-negative.}
    Every learning rule $\A$ with $\k_\A = \k_\H$ has non-negative regret:
    there exists $p'=p'(\H)>0$ so that for every $p\leq p'$ and every $T$ there exists an input sequence of length $T$
    such that $\mathtt{opt_\H} \leq \lfloor p\cdot T\rfloor$ but the expected number of mistakes made by~$\A$ is~$\geq \lfloor p\cdot T\rfloor$.
    \item {\bf Unbounded Regret.}\label{item:agn3}
    If in addition to Item~\ref{item:agn2}, there are $h',h''\in \H$ and an input $x$ such that $h'(x),h''(x)$ are distinct lists of size $\k_\H$ then every learning rule $\A$ with $\k_\A = \k_\H$  has regret at least $c\cdot \sqrt{T}$, where $c=c(\H)$ does not depend on $T$.
\end{enumerate}
\end{theorem}

Item~\ref{item:agn2} is tight in the sense that there is a class $\H$ satisfying the assumption in Item~\ref{item:agn2} which can be learned by a learning rule $\A$ with $\k_\A=\k_\H$ that has a \underline{non-positive} regret.
Indeed, let $\X=\mathbb{N}$ and $\Y=\{0,1,2,3\}$, and consider $\H=\{0,1\}^\mathbb{N}$. Thus, $\k_\H=2$ and all functions in $\H$ map each $x\in \X$ to a single output in $\{0,1\}$ (and hence the additional assumption in Item~\ref{item:agn3} does not apply).
Finally, notice that the learning rule $\A$ which always predicts with the list $\{0,1\}$ achieves regret $\leq 0$.

On the other hand, there are also such classes (i.e.\ that satisfy the assumption in Item~\ref{item:agn2} but not the additional assumption in Item~\ref{item:agn3}) 
for which the optimal regret does scale like $\Omega(\sqrt{T})$.
One such example is $\H=\{0,1\}^\X\cup \{2,3\}^\X$. 
It will be interesting to refine the above theorem and characterize, for every class $\H$, the optimal regret achievable by learning rules $\A$ for which $\k_\A=\k_\H$.

\medskip

The proof of Theorem~\ref{thm:negreg} appears in Section~\ref{sec:negregproof}. It combines a variety of techniques: in Item~\ref{item:agn1} the idea is to combine two algorithms: one which uses lists of size $\k_\H$ and has vanishing regret, and another that uses additional $k-\k_\H>0$ labels that were not chosen by the first algorithm in a way that achieves the overall regret. In fact, we show that for the second algorithm, we can simply pick the $k-\k_\H$ labels uniformly at random from the remaining labels.
The proof of Item~\ref{item:agn2} follows by a randomized construction of a hard input sequence which witnesses the nonngegative regret. In particular we start by taking a shattered $\k_\H$-tree of depth $T$ (such a tree exists for all $T$ because $\L_{\k_\H-1}(\H)=\infty$), then pick a sequence corresponding to a random branch in the tree, and carefully modify this random sequence in $p\cdot T$ random locations.
The proof of Item~\ref{item:agn3} also follows by a randomized construction of a hard input sequence. The argument here generalizes the classical lower bound of $\Omega(\sqrt{T})$ for any class $\H$ which consists of at least $2$ functions. However, the proof here is a little bit more subtle.

\subsection{Online List Sauer-Shelah-Perles Lemma}
A central tool in the derivation of our upper bound in the agnostic setting
is a covering lemma, which takes a realizable case learner for a pattern class $\P$ 
and uses it to construct a small collection of online/adaptive functions which cover $\P$.
We begin with introducing some definitions and notations.

\paragraph{Online Functions.}
An online (or adaptive) function is a mapping $f:\X^\star\to \Y$.
That is, its input is a finite sequence of points $x_1,\ldots, x_{t-1},x\in \X$
and its output is a label $y\in \Y$ which we denote by 
$y=f(x ; x_1,\ldots, x_{t-1})$.
We interpret this object as a function with memory:
the points $x_1,\ldots x_{t-1}$ are the points that $f$ has already seen
and stored in its memory, and $y$ is the label it assigns to the current point $x$.
Similarly, a $k$-list online function is a mapping $f:\X^\star\to \binom{\Y}{k}$.

\begin{definition}[Online Coverings]
Let $\P$ be a pattern class and let $T\in\mathbb{N}$. 
We say that a family $\F$ of $k$-list online functions is a $T$-cover of $\P$
if for every pattern $\{(x_t,y_t)\}_{t=1}^T\in\P$ of length $T$ there exists a $k$-list online function $f\in \F$ such that
\[(\forall t\leq T): y_t \in f(x_t ; x_1,\ldots, x_{t-1}).\]
\end{definition}

\begin{prop}[List Sauer-Shelah-Perles Lemma]\label{prop:ssp}
Let $\P$ be a pattern class with $\L_k(\P)=d < \infty$, and let $L=\lvert \Y\rvert$. 
Then, for every $T\in\mathbb{N}$ there is a family $\F$ of $k$-list online functions which is a $T$-cover of $\P$  such that
\[\lvert \F\rvert \leq \sum_{i=0}^{d}\binom{T}{i}\left(\ceil*{\frac{L-k}{k}}k\right)^i \approx \left(e\cdot \frac{T}{d}L\right)^d.\]
\end{prop}

The construction of the covering family $\F$ follows a similar approach like in~\citet{Ben-david09agnostic}. In particular it relies on our lazy List SOA (lazy in the sense that it only changes its predictor when making a mistake). Then, each $k$-list online function in $\F$ simulates the List SOA, on all but at most $d$ steps, aiming to fully cover all patterns in $\P$.

\section{Proofs}\label{sec:proofs}

\subsection{List Perceptron}\label{sec:percproof}

\begin{proof}[Proof of Theorem~\ref{thm:listperc}]
We follow the classical analysis of the Perceptron by upper and lower bounding the sum $\sum_y \|w_y\|^2$ at the end of round $T$.
This sum only changes on rounds $t$ in which the algorithm makes a mistake.
In such rounds, the vectors $w_y$ are updated to $w'_y$ as follows
\begin{align*}
w_{y_t}' &=  w_{y_t} + k x_t,\\
w_y' &= w_y - x_t. \tag{for all $y\in P_t$}
\end{align*}
Thus, the change in the potential function from round $t$ to $t+1$ is upper bounded as follows
\begin{align*}
\sum_{y\in \Y} \|w_y'\|^2 - \sum_{y\in \Y} \|w_y\|^2 &=  \Bigl(\|w'_{y_t}\|^2 - \|w_{y_t}\|^2\Bigr) + \Bigl(\sum_{y\in P_t} \|w_y'\|^2 - \|w_y\|^2\Bigr)\\
&= k^2\|x_t\|^2 + 2k(w_{y_t}\cdot x_t ) + \sum_{y\in P_t}(\|x_t\|^2 - 2w_y\cdot x_t)\\
&= k(k+1)\|x_t\|^2 +2\sum_{y\in P_t} (w_{y_t} - w_y)\cdot x_t\\
&\leq k(k+1)\|x_t\|^2 \tag{$w_{y_t}\cdot x_t \leq  w_y\cdot x_t$}\\
&\leq k(k+1)R^2.  \tag{$\|x_t\|\leq R$}
\end{align*}
Thus, if we denote by $M$ the total number of mistakes that the Perceptron makes on $S$ then the potential function at time $T$ satisfies:
\begin{equation}\label{eq:ubp}
\sum_{y\in \Y} \|w_y\|^2 \leq M\cdot k(k+1)R^2.
\end{equation}
We now lower bound the potential function.
Let $w^\star_y$ for $y\in \Y$ denote the unit vectors that witness that $S$ is linearly separable with $k$-margin at least~$\gamma_k$. Consider the sum of inner products $\sum_{y\in \Y} w_y\cdot w_y^\star$.
This sum also changes only on rounds $t$ when a mistake occurs, hence
\begin{align*}
\sum_{y\in \Y} w'_y\cdot w^\star_y  - \sum_{y\in \Y}w_y\cdot w^\star_y &= kx_t\cdot w^\star_{y_t} -\sum_{y\in P_t} x_t\cdot w^\star_y\\
    &=\sum_{y\in P_t}  (w^\star_{y_t} - w^\star_y)x_t
    \geq \gamma_k. \tag{because $\|P_t\|=k$}
\end{align*}
Thus, after round $T$ we have $M\gamma_k \leq \sum_{y\in \Y} w_y\cdot w_y^\star$ and hence:
\begin{align*}
M\gamma_k &\leq \sum_{y\in \Y} w_y\cdot w_y^\star
        \leq \sum_{y\in \Y} \|w_y\|\cdot \|w_y^\star\| \tag{Cauchey-Schwartz}\\
        &\leq \sqrt{\sum_{y\in \Y} \|w_y\|^2}\cdot\sqrt{\sum_{y\in \Y} \|w_y^\star\|^2}=\sqrt{\sum_{y\in \Y} \|w_y\|^2}. \tag{Cauchey-Schwartz}     
\end{align*}
Hence, $\sum_{y\in \Y} \|w_y\|^2 \geq M^2\cdot \gamma_k^2$, which in combination with Equation~\ref{eq:ubp} yields
$M\leq k(k+1)\frac{R^2}{\gamma_k^2}$ as required.

\end{proof}

\subsection{Realizable Case}
The qualitative statement of Theorem \ref{thm:realchar} is just a corollary of the quantitative statement given by Theorem \ref{thm:realquant}.
\begin{proof}[Proof of Theorem~\ref{thm:realquant}]
The lower bound follows directly from the definition of mistake trees and $\L_k(\P)$. 
Let $\A$ be a deterministic $k$-list learner, and let $\T$ be a shattered complete ($k+1$)-ary mistake tree of depth $n$. We will build an input sequence by following the adversary on the shattered tree $\T$. Starting from the root, we will choose a label for each node $x_i$ in the root-to-leaf path, associated with one edge from its $k+1$ out-going edges, so that the chosen label will not be in the predicted set of the learner $\A$. We can always do that because the predicted set in of size $k$ and for each node in the tree there are $k+1$ out-going edges (with different labels). With this construction we will get a sequence of size $n$ which is realizable by $\P$ and forces $\A$ to always err and hence to make $n$ mistakes on the input-sequence. Therefore, for finite $(k+1)$-ary Littlestone dimension $d$, we can construct a realizable sequence of length $d$ such that $\A$ makes at least $d$ mistakes on $S$, and for infinite $(k+1)$-ary Littlestone dimension, we can construct a realizable sequence of length $n$, for each $n\in \mathbb{N}$ such that $\A$ makes at least $n$ mistakes on $S$.

\begin{corollary}
    \[\mathtt{M}_k(\P)\geq \L_k(\P).\]
\end{corollary}

The lower bound for randomized learners is achieved in a very similar way. We still follow a root-to-leaf path in a shattered tree, but now, instead of choosing the unchosen label by the learner, we choose a {\it random} label, so that at each step, the algorithm makes a mistake with probability $\frac{1}{k+1}$.

\begin{corollary}
\[\R_k(\P) \geq \frac{1}{k+1}\L_k(\P).\]
\end{corollary}

The upper bound relies on an adaptation of Littlestone's Standard Optimal Algorithm (SOA) to the list learning setting and to pattern class, described in Figure~\ref{fig:ListSOA}.


\begin{claim}\label{clm:upperSOA}
Let $\P$ be a pattern class such that $\L_k(\H)<\infty$ and let $x\in \X$. Then the number of $y\in \Y$ such that $\L_k(\P_{x\rightarrow y})=\L_k(\P)$ is at most $k$.
\end{claim}

\begin{proof}
Assume by contradiction that there exists a subset $I\subseteq\Y, |I|=k+1$ such that $\forall y\in I$, $\L_k(\P_{x\rightarrow y})=\L_k(\P)=:d$. Then we will construct the following tree. Let $\{\T_y\}_{y\in I}$ be complete $(k+1)$-ary mistake trees with depth $d$ shattered by $\{\P_{x\rightarrow y}\}_{y\in I}$ respectively. Construct a tree with $x$ as a root, where its outgoing edges are labeled with the labels in $I$ and $\{\T_y\}_{y\in I}$ are the subtrees under $x$ where the subtree under each outgoing edge with label $y\in I$ is $\T_y$. With this construction we get a complete $(k+1)$-ary mistake tree of depth $d+1$ which is shattered by $\P$ which contradicts with the assumption that $d$ is the maximal depth of a complete $(k+1)$-ary mistake tree which is shattered by $\P$.
\end{proof}

\begin{claim}
The List SOA makes at most $\L_k(\P)$ mistakes on every realizable sequence.
\end{claim}

\begin{proof}
Note that at each time step where the List SOA makes a mistake, the $(k+1)$-ary Littlestone dimension of $V$ is decreasing. This is because of Claim~\ref{clm:upperSOA} combined with the construction of the predicted list $P_t$ to contain the $k$ labels with maximal value $\L_k\left(V\right)$. Since the $(k+1)$-ary Littlestone dimension is bounded below by 0, it can only decrease at most $\L_k(\P)$ times. 
Therefore, the number of mistakes the algorithm makes is bounded above by $\L_k(\P)$.
\end{proof}

\begin{corollary}
\[\M_k(\P) \leq \L_k(\P).\]
\end{corollary}
 The upper bound of $\L_k(\P)$ achieved by the List SOA, completes the characterization of list learnability for realizable sequences, and shows that the optimal mistake bound of deterministic learner is exactly the $(k+1)$-ary Littlestone dimension $\L_k(\P)$.

 For randomized learners, the lower and upper bounds of the optimal number of mistakes, provided in Theorem \ref{thm:realquant} are tight;
 \begin{itemize}
     \item  There is a $k$-list learnable pattern class $\P$, such that for every $k$-list learner $\A$, there exists a realizable sequence $S$, with $\M(\A;S) \geq \L_k(\P)$.
    \item  There is a $k$-list learnable pattern class $\P$, and a $k$-list learner $\A$, such that for every realizable sequence $S$,  $\M(\A;S) \leq \frac{1}{k+1}L_k(\P)$.
 \end{itemize}
 
 See Appendix~\ref{app:randexamples} for the construction of these extremal classes, and the analysis of their optimal mistake bounds.

\end{proof}

\subsection{Online Sauer-Shelah-Perles Lemma}\label{sec:sspproof}
\begin{proof}[Proof of Proposition~\ref{prop:ssp}]
The construction of the covering family $\F$ follows a similar approach like in~\citet{Ben-david09agnostic}. In particular it relies on our lazy List SOA (lazy in the sense that it only changes its predictor when making a mistake). Each $k$-list online function in $\F$ simulates the List SOA, on all but at most $d$ steps, aiming to fully cover all patterns in $\P$.

We now describe the family $\F$.
Since our goal is to cover realizable sequences of length $T$, it suffices to define each $f\in \F$ on input sequences in $\X^\star$ of size at most $T$.\footnote{If $t > T$ then set $f(x_t; x_{1},\ldots, x_{t-1})= 0$ for all $f\in \F$ and all $x_1,\ldots, x_t\in \X$.}
In what follows, it is convenient to fix a linear order on the label set $\Y$ (e.g.\ one can think of $\Y$ as $\Y=\{0,\ldots, L-1\}$).
Also, for every $A\subseteq \Y$, fix a covering $\{A_1,\ldots, A_{m}\}$ of $A$ into $m=\lceil \lvert A\rvert / k \rceil$ sets,
where each set has size $k$ (e.g., the first $k$ elements in $A$ form the first part, the second $k$ elements form the next part, and so on; if the last part has size $l$ less than $k$, then complete it with the first $k-l$ elements to get a size of $k$). 
 We will refer to this fixed cover as \emph{the canonical $k$-cover of $A$.}
 Each $f\in \F$ is parameterized by a subset $|I|\subset [T]$ of size $|I|\leq d$ and a vector of pairs $\bigl((i_t,j_t)\bigr)_{t\in I}$,
 where $i_t\in \{1,\ldots, k\}$, and $j_t\in\{1,\ldots, \lceil\frac{L-k}{k}\rceil\}$.
Since there are $\lceil{\frac{L-k}{k}}\rceil k$ ways to pick a given pair $(i_t,j_t)$, we get
 \[\lvert \F\rvert = \sum_{i=0}^{d}\binom{T}{i}\Bigl(\ceil*{\frac{L-k}{k}}k\Bigr)^i.\]
 Each such function $f$ is defined as follows:

\begin{tcolorbox}
\textbf{Parameters:} $T \in \mathbb{N}$, $\big(I, \{(i_t,j_t)\}_{t\in I}\big)$, where $I\subseteq [T]$ and $i_t\leq k, 
j_t\leq \lceil\frac{L-k}{k}\rceil$.

\textbf{Input:} An input sequence of unlabeled examples $\vec x= \{x_t\}_{t=1}^T$ of length $T$.

\textbf{Initialize: $S=\emptyset$.}
For $t=1,2,\ldots,T$
\begin{enumerate}
\item Get $x_t$ and let $P_t = \text{ListSOA}(x_t ; S)$.
\item If $t\notin I$ then:
\begin{itemize}
\item Set $f(x_t ; \vec x_{<t})=P_t$.
\item Set $y_t$ to be any label in $f(x_t ; \vec x_{<t})$.
\end{itemize}

\item If $t\in I$ then:
\begin{itemize}
\item Set $f(x_t ; \vec x_{<t})$ to be set number $j_t$ in the canonical $k$-cover of $\Y\setminus P_t$.
\item Set $y_t$ to be the $i_t$'th element in $f(x_t ; \vec x_{<t})$.
\end{itemize}
 
\item Update $S \leftarrow S \circ (x_t,y_t)$.
\end{enumerate}
\end{tcolorbox}

 Finally, it remains to show that $\F$ covers $\P$. That is, to show that for every $S \in \P$ there exists $f\in \F$ such that $y_t\in f(x_t ; x_1,\ldots, x_{t-1})$ for all $(x_t,y_t)\in S$.
To this end, imagine a simulation of the List SOA on the sequence $S$, and denote by $I=I(S)$ the set of all indices where the List SOA made a mistake on $S$. Notice that $\lvert I\rvert \leq d$, since $S$ is realizable by $\P$. For each $t\in I$, let $j_t$ be an index of a set $A_t$ in the canonical $k$-cover of $\Y\setminus \text{ListSOA}(x_t; S_{<t} )$ such that $y_t\in A_t$, and let $i_t$ be the ordinal of $y_t$ in $A_t$.\footnote{I.e.\ there are $i_t-1$ elements in $A_t$ that are smaller than $y_t$ with respect to the assumed ordering on $\Y$.} Consider the function $f\in \F$ parameterized by
$(I, \{(i_t,j_t)\}_{t\in I})$. By the laziness property of the List SOA algorithm, it follows that
$f(x_t ; \vec x_{<t}) = \text{ListSOA}(x_t; S_{<t})$ on each time step $t\notin I$. By construction, on time steps $t\in I$, the list $f(x_t ; \vec x_{<t})$ contains the label $y_t$. Thus, in total $f$ covers $S$ as required.
\end{proof}

\subsection{Agnostic Case}\label{sec:agnproof}
The proof of Theorem~\ref{thm:agnchar} consists of an upper bound statement and a lower bound statement:
\begin{itemize}
    \item \textbf{Upper bound.} finite $\L_k(\H)$ $\rightarrow$ $\H$ is agnostically $k$-list learnable.
    \item \textbf{Lower bound.} infinite $\L_k(\H)$ $\rightarrow$ $\H$ is not agnostically $k$-list learnable.
\end{itemize}

The proof of the upper bound proof follows a similar approach as in the agnostic-to-realizable reduction by~\citet{Ben-david09agnostic,Daniely15Multiclass}: the idea is to cover the set of all sequences realizable by $\H$ using a small set of adaptive experts, and then run a vanishing regret algorithm on top of these experts.
We break the proof into a few propositions.
Proposition~\ref{prop:agnMWalg} shows that each finite class of $k$-list online functions is agnostically $k$-list learnable, with regret $O(\sqrt{T\ln n})$, where $n$ is the number of functions in the class. 
This follows by a standard application of vanishing regret algorithms, here we use \emph{Multiplicative Weights}. 
In Proposition \ref{prop:agnw/oT} we show that a general multi-labeled hypothesis class with finite $(k+1)$-ary Littlestone dimension is $k$-list learnable (with a matching upper bound), relying on Proposition~\ref{prop:agnMWalg} for finite class of experts, and the list version of SSP lemma~\ref{prop:ssp} from previous section.

\subsubsection{Finite Class of Experts}

\begin{prop}\label{prop:agnMWalg}   
Let $\H$ be a class of $k$-list online functions of size $\lvert \H\rvert=n$. 
Then, there exists a $k$-list learner $\A$ for $\H$ with regret $\R(\A;S)\leq 2\sqrt{T\ln{n}}$ for every input sequence $S=\{(x_i,y_i)\}_{i=1} ^{T}$.

\end{prop}

We will construct the algorithm $\A$ using the randomized multiplicative weights method, so the desired algorithm will be a randomized algorithm. A randomized $k$-list algorithm, is an algorithm which for each input sample returns some probability distribution over all subsets of $\Y$ of size $k$. 
For our convenience in describing and analyzing the algorithm, we will use a vector representation of the prediction $P_t$, where the prediction $P_t$ is a vector of length $L=\lvert\Y\rvert$, and each entry $i$ represents the probability that $i$ is in the predicted set.
For a deterministic $k$-list algorithm, the value of each entry in the vector representation of $P_t$ will be 0 or 1, and the number of 1's will be $k$.
For a randomized $k$-list algorithm, the value of each entry in the vector representation of $P_t$ will be between 0 and 1, and the summation over all entries will be $k$.
With this point of view we can also look at a randomized algorithm as a deterministic algorithm, where the loss function (i.e. the probability to make a mistake) will be $l_t=1-P_t(y_t)$ 
, which is exactly the probability that $y_t$ will not be in the randomized-chosen predicted set $P_t$.

\begin{proof}
The proof uses the multiplicative weights algorithm with parameter $\gamma\in(0,1)$, described in Figure~\ref{fig:MWAlg}.
This algorithm runs over a finite class of experts, 
where each expert is a $k$-list online function.

\begin{figure}
\begin{tcolorbox}
\begin{center}
{\bf Multiplicative Weights Algorithm}
\end{center}

\textbf{Parameter:} $\gamma\in (0,1)$.

\textbf{Input:} A finite class $\H$ of experts.

\textbf{Initialize:} $\forall f\in \H$ initialize $W_1(f)=1$.

For $t=1,2,\ldots$
\begin{enumerate}
    \item Receive unlabeled example $x_t\in \X$.
    \item Predict $P_t=\sum_{f\in \H} \frac{W_t(f)}{W_t} \cdot f(x_t ; x_1,\ldots, x_{t-1})$ where $W_t:=\sum_{f\in \H} W_t(f)$.
    \item Receive correct label $y_t\in \Y$.
    \item $\forall f\in \H$, update $W_{t+1}(f) =
         \begin{cases}
         W_t(f) & y_t\in f(x_t) \\
         (1-\gamma )W_t(f) & y_t\notin f(x_t)
         \end{cases}
         $
\end{enumerate}
\end{tcolorbox}
\caption{Multiplicative Weights Algorithm}
\label{fig:MWAlg}
\end{figure}

The computation of the prediction $P_t$ uses the vector representation of each expert $f\in \H$, so that $P_t$ is a convex combination of those vectors and hence a legal vector representation of a randomized $k$-list algorithm's prediction i.e. each entry is between 0 and 1, and the summation over all entries is~$k$.
The analysis of the multiplicative weights algorithm does not care of the type of the experts in class. In particular, the analysis does not depend neither on the number of labels $L$, nor on the list's size $k$.
Let $S=\{(x_t,y_t)\}_{t=1} ^{T}$ be any input sequence, let $L_T = \sum_{t=1}^{T}l_t$ be the accumulated loss until time $T$, and set $\mathtt{opt_\H}:= \min_{f\in \H} \M(f;S)$. The sum of all weights at time $t+1$ equals to the following expression:
\[W_{t+1}=l_t W_t (1-\gamma)+(1-l_t) W_t=W_t (1-\gamma l_t)\]
Hence, if we denote the best expert in $\H$ by $f^*$ and note that $W_1=n$, we get the following relation:
\[(1-\gamma)^{\mathtt{opt_\H}}=W_T(f^*)\leq W_T=n\prod_{t=1}^T(1-\gamma l_t) \leq n\exp{(-\gamma \sum_{t=1}^T l_t)}=n\exp{(-\gamma L_T)}\] 
Thus, \[(1-\gamma)^{\mathtt{opt_\H}}\leq n\exp{(-\gamma L_T)}\]
\begin{claim}\label{claim:agnIneq} 
$ 1-x\geq \exp(-x-x^2) ~\forall x\in [0, \frac{1}{2}]$
\end{claim}
\begin{proof}
By the Taylor series, it is known that $\forall x\in [0,1)$, $\frac{1}{1-x}=1+x+x^2+x^3+\ldots$, hence we get:
\begin{align*}
    1-x &=\exp(\ln(1-x))=\exp(-\int\frac{1}{1-x})=\exp(-\int (1+x+x^2+\ldots)) \\
    &=\exp(-(x+\frac{x^2}{2}+\frac{x^3}{3}+\ldots)) \underset{x\leq \frac{1}{2}}{\geq} \exp(-(x+\frac{x^2}{2}+\frac{x^2}{2}))
\end{align*}

\end{proof}
Using Claim~\ref{claim:agnIneq} with $x=\gamma$, we get:
\[\exp({\mathtt{opt_\H}}(-\gamma-\gamma^2))\leq (1-\gamma)^{\mathtt{opt_\H}}\leq n\exp(-\gamma L_T)\]
\[\Rightarrow{} -{\mathtt{opt_\H}}(\gamma+\gamma^2))\leq \ln{n}-\gamma L_T\]
\[\Rightarrow{} \gamma(L_T-{\mathtt{opt_\H}})\leq \ln{n}+\mathtt{opt_\H}\cdot \gamma^2\]
\[\Rightarrow{} \R(\A;S)=L_T-{\mathtt{opt_\H}}\leq \frac{\ln{n}}{\gamma}+\mathtt{opt_\H}\cdot \gamma\leq \frac{\ln{n}}{\gamma}+T\gamma\]
This gives us an upper bound for the regret, depending on $\gamma$, and by setting $\gamma=\sqrt{\frac{ln{n}}{T}}$ we will get the desired bound from the theorem:
\[\R(\A;S)=L_T-{\mathtt{opt_\H}}\leq 2\sqrt{Tln{n}}.\]
Note that the above result is valid only for large values of $T$ when $\gamma=\sqrt{\frac{ln{n}}{T}}\leq \frac{1}{2}$. But we can overcome this obstacle easily. If $\sqrt{\frac{ln{n}}{T}}> \frac{1}{2}$ then $T<4\ln{n}$. Hence $T^2<4T\ln{n}$ and by taking the square root of both sides, $T<2\sqrt{T\ln{n}}$. Since the regret is bounded by the length of the input sequence $T$, we will get that if $\sqrt{\frac{ln{n}}{T}}> \frac{1}{2}$ then,
\[\R(\A;S)\leq T\leq 2\sqrt{Tln{n}}.\]
\end{proof}

\subsubsection{General Class}

The previous section showed that any finite class of $k$-list online experts is agnostic online learnable in the $k$-list setting with a bound on the regret of $2\sqrt{T\ln{n}}$, where $n$ is the number of experts and $T$ is the length of the input sequence. 
By using the list SSP lemma to cover a (possibly infinite) hypothesis class $\H$ by a finite set of $k$-list online functions, we can generalize this result to general classes with no restriction on the size of the class. 


\begin{prop}\label{prop:agnw/oT}
Let $\H$ be a (possibly multi-labeled) hypothesis class, and let $k\in\mathbb{N}$ such that $\L_k(\H)=d<\infty$.  
Then, there exists a $k$-list learner $\A$ whose regret with respect to $\H$ satisfies 
\[\R(\A;S)<\frac{2\sqrt{2}}{\sqrt{2}-1}\sqrt{dT+dT\ln\left(\frac{T}{d}\ceil*{\frac{L-k}{k}}k\right)}\]
for every input sequence $S$, where $T$ denotes the size of $S$.
\end{prop}

We prove Proposition~\ref{prop:agnw/oT} in two steps: first we prove Lemma~\ref{lemma:agnw/T} which is slightly weaker than 
Proposition~\ref{prop:agnw/oT} because the learning rule there depends on the horizon $T$.
Then, we use a standard doubling trick argument to obtain a learning rule that achieves vanishing regret simultanously for all $T$.

\begin{lemma} \label{lemma:agnw/T}
Let $\H$ be a (possibly multi-labeled) hypothesis class, let $k\in\mathbb{N}$ such that $\L_k(\H)=d<\infty$, and let $T\in\mathbb{N}$.  Then, there exists a $k$-list learner $\A$ whose regret with respect to $\H$ satisfies 
\[\R(\A;S)<2\sqrt{dT+dT\ln\Bigl(\frac{T}{d}\ceil*{\frac{L-k}{k}}k\Bigr)},\]
for every input sequence $S$ of size $T$.
\end{lemma}

\begin{proof}
Let $\F$ be a family of $n \leq \binom{T}{\leq d}\Bigl(\ceil*{\frac{L-k}{k}}k\Bigr)^d$ $k$-list online functions which is a $T$-cover of $\P(\H)$. Namely, that covers all realizable sequences by $\H$ of length $T$. Apply the multiplicative weights algorithm $\A$ described in the proof of Proposition~\ref{prop:agnMWalg} on the class $\F$, and get that 
\begin{align*}
\R(\A;S)&\leq 2\sqrt{T\ln{n}} \\
&= 2\sqrt{T\ln\Bigl(\binom{T}{\leq d}\Bigl(\ceil*{\frac{L-k}{k}}k\Bigr)^d\Bigr)} \\
&\leq 2\sqrt{T\ln\Bigl(\Bigl(\frac{eT}{d}\Bigr)^d\Bigl(\ceil*{\frac{L-k}{k}}k\Bigr)^d\Bigr)} \\
&=2\sqrt{dT+dT\ln\Bigl(\frac{T}{d}\ceil*{\frac{L-k}{k}}k\Bigr)}=:r(T)
\end{align*}
 where the regret is with respect to the class $\F$. This means that
 \[\M(\A;S) -  \min_{f\in \F} \M(f;S) \leq r(T).\] Since $ \min_{f\in \F} \M(f;S)\leq  \min_{h\in \H} \M(h;S)$ we get that 
 \[\M(\A;S) -  \min_{h\in \H} \M(h;S)\leq \M(\A;S) -  \min_{f\in \F} \M(f;S)\leq r(T).\] Hence,
 \[\R(\A;S)\leq r(T)\]
 with respect to the class $\H$.
\end{proof}

\begin{proof}[proof of Proposition \ref{prop:agnw/oT}]
We will use the algorithm $\A$ described on the proof of Lemma~\ref{lemma:agnw/T}, which knows the length of the input sequence, and apply it on subsequences of the original input sequence of length $T$ as follows:
The first subsequence $S_1$ will be the first sample in $S$: $(x_1,y_1)$, the second subsequence $S_2$ will be the next two samples: $(x_2,y_2), (x_3,y_3)$, the third subsequence $S_3$ will be the next four samples: $(x_4,y_4), (x_5,y_5), (x_6,y_6), (x_7,y_7)$ and so on. We will get a sequence of subsequences $\{S_i\}_{i\geq1}$ with increasing lengths: $\{T_i\}_{i\geq1}=1,2,4,8,\ldots$ such that the length of the $i$'th sequence $S_i$ will be $\lvert S_i\rvert=T_i=2^{i-1}$. 
\footnote{The length of the last sequence may be smaller, but it can only improve the bound.} 
On each of these subsequences with a known length, we can now apply the algorithm $\A$, to get the following bound on the regret:
\[2\sqrt{dT_i+dT_i\ln\left(\frac{T_i}{d}\ceil*{\frac{L-k}{k}}k\right)}.\]
For each step $i$, we got a bound on the regret for a specific pattern $h_i\in \H$, so 
\[\min_{h\in \H} \M(h;S)\geq \sum_{i=1}^{\lfloor \log{T}\rfloor +1}\min_{h\in \H} \M(h;S_i),\]
since the optimal ``global" hypothesis $h$ may make more mistakes on $S_i$ then the optimal ``local" hypothesis $h_i$. Hence,
\[\R(\A;S)\leq \sum_{i=1}^{\lfloor \log{T}\rfloor +1}\R(\A_i;S_i),\]
where $\A_i$ is the algorithm applied on the subsequence $S_i$, with the assumption of length $T_i$, and with some calculations we will get
\begin{align*}
\R(\A;S)&\leq \sum_{i=1}^{\lfloor \log{T}\rfloor +1}\R(\A_i;S_i) \\
&=    \sum_{i=1}^{\lfloor \log{T}\rfloor +1}2\sqrt{dT_i+dT_i\ln\left(\frac{T_i}{d}\ceil*{\frac{L-k}{k}}k\right)} \\
&=    \sum_{i=1}^{\lfloor \log{T}\rfloor +1}2\sqrt{d2^{i-1}+d2^{i-1}\ln\left(\frac{2^{i-1}}{d}\ceil*{\frac{L-k}{k}}k\right)} \\
&\leq \sum_{i=1}^{\lfloor \log{T}\rfloor +1}2\sqrt{d2^{i-1}+d2^{i-1}\ln\left(\frac{2^{\log{T}}}{d}\ceil*{\frac{L-k}{k}}k\right)} \\
&=    2\sqrt{d+d\ln\left(\frac{T}{d}\ceil*{\frac{L-k}{k}}k\right)}\sum_{i=1}^{\lfloor \log{T}\rfloor +1}2^{i-1} \\
&=    2\sqrt{d+d\ln\left(\frac{T}{d}\ceil*{\frac{L-k}{k}}k\right)}\frac{\sqrt{2}^{\lfloor \log{T}\rfloor +1}-1}{\sqrt{2}-1} \\
&\leq 2\sqrt{d+d\ln\left(\frac{T}{d}\ceil*{\frac{L-k}{k}}k\right)}\frac{\sqrt{2}^{\lfloor \log{T}\rfloor +1}}{\sqrt{2}-1} \\
&\leq 2\sqrt{d+d\ln\left(\frac{T}{d}\ceil*{\frac{L-k}{k}}k\right)}\frac{\sqrt{2T}}{\sqrt{2}-1} \\
&= \frac{2\sqrt{2}}{\sqrt{2}-1}\sqrt{dT+dT\ln\left(\frac{T}{d}\ceil*{\frac{L-k}{k}}k\right)}. \\
\end{align*}
\end{proof}

The following Proposition shows that a class with infinite $(k+1)$-ary Littlestone dimension is not $k$-list learnable in the agnostic setting.

\begin{prop}\label{prop:agnlower}
Let $\H$ be a (multi-labeled) hypothesis class with infinite $(k+1)$-ary Littlestone dimension $\L_k(\H)=\infty$. Then for each (maybe randomized) $k$-list algorithm $\A$ and $T\in \mathbb{N}$ there exists an input sequence $S$ of length $T$ such that the regret is \[\R(\A;S)= \Omega(T).\]
\end{prop}

\begin{proof}
The proof repeats the same argument as in the lower bound proof of $\L_k(\P)$ for randomized learners over the number of mistakes;
let $\A$ be a $k$-list online learner and let $T\in \mathbb{N}$. Since $\L_k(\H)=\infty$, there is a complete $(k+1)$-ary mistake tree of depth $T$, which is shattered by $\H$. 
Follow a root-to-leaf path in the shattered tree, and choose one of the $k+1$ labels associated to the $k+1$ outgoing edges from the current node in the path, uniformly at random. At each step, the learner $\A$ will make a mistake with probability $\frac{1}{k+1}$. The sequence $S$ of length $T$ generated by this process, is realizable by $\H$, which means $\min_{h\in \H} \M(h;S)=0$. Further, $\A$ makes $\frac{1}{k+1}T$ mistakes on $S$ in expectation. Thus, the regret will be $\frac{1}{k+1}T$ which means $\R(\A;S)= \Omega(T)$.
\end{proof}

\begin{proof} [proof of Theorem \ref{thm:agnchar}] The proof of the upper bound on the regret was given in a series of propositions: Proposition~\ref{prop:agnMWalg}, Lemma~\ref{lemma:agnw/T} and Proposition~\ref{prop:agnw/oT}, while the proof of the lower bound was given by Proposition~\ref{prop:agnlower}.

\end{proof}


    

\subsection{Proof of Theorem \ref{thm:negreg}}\label{sec:negregproof}

\begin{proof} 
Let $\H$ be a multi-labeled hypothesis class, $\lvert\H\rvert>1$, and let $1\leq \k_\H\leq \lvert \Y\rvert$ be the minimal number such that $\L_{\k_\H}(\H)<\infty$ \footnote{Notice that the minimum always exists if $\Y$ is finite since $\L_{\lvert \Y\rvert}(\H)<\infty$.}.
By Theorem~\ref{thm:agnchar}, an agnostic list-learning rule $\A$ for $\H$ exists if and only if $\k_\H \leq \k_\A$, where $\k_\A$ is the size of the lists used by $\A$. 

\begin{proof}[Proof of Item~\ref{item:agn1} (Negative Regret)]

Let $\k_\H < k \leq \lvert \Y\rvert$. We will show that there exists a learning rule $\A$ with $\k_\A=k$ which learns $\H$ with a negative regret in the following sense: its expected regret on every input sequence of length $T$ is at most
\[(p-1)\cdot \mathtt{opt_\H} + c\cdot \bigl(\sqrt{{T\ln T}}\bigr),\]
where $p=p(\H)<1$, $c=c(\H)$ both do not depend on $T$, and $\mathtt{opt_\H}$ is the number of mistakes made by the best function in $\H$.
The idea is to use a list-learning rule, denoted by $\mathcal{B}$, that uses lists of size~$\k_\H$ and has vanishing regret ($O(\sqrt{d\cdot T\ln T})$) where $d=\L_{\k_\H}(\H)$ (such a learning rule exists by Theorem~\ref{thm:agnchar}), and in each time-step add to its list $\k_\A-\k_\H$ labels that are chosen uniformly at random from the remaining labels. 
To analyze the regret, notice that on every fixed input sequence $S=\{(x_t,y_t)\}_{t=1} ^{T}$ of length $T$, the accumulated loss $\A$ suffers is equal to the accumulated loss $\mathcal{B}$ suffers times $p<1$, where $p=\frac{\lvert\Y\rvert - k}{\lvert\Y\rvert - \k_\H}$. 
Indeed, in every time-step $t$, let $\ell^\mathcal{B}_t$ denote the expected loss of $\mathcal{B}$ in this round, that is, $\ell^\mathcal{B}_t$ is the probability that the random list of size $\k_\H$ does not include $y_t$, and also let $\Y_{add}$ denote the (random) set of labels that are added to the list $\mathcal{B}$ predicts (so $\lvert\Y_{add}\rvert=k-\k_\H$). Then, 
\begin{align*}
  \ell^\A_t
  &=\ell^\mathcal{B}_t\cdot\Pr\Bigl[y_t \notin \Y_{add} \vert y_t \notin \mathcal{B}\Bigl(x_t;\{(x_i,y_i)\}_{i<t}\Bigr)\Bigr] \\
  &=\ell^\mathcal{B}_t\left(1-\frac{k-\k_\H}{\lvert\Y\rvert - \k_\H}\right)
  =\ell^\mathcal{B}_t\left(\frac{\lvert\Y\rvert - k}{\lvert\Y\rvert - \k_\H}\right)
  =\ell^\mathcal{B}_t\cdot p.
\end{align*}
And by summing over the whole input sequence $S$, we get
\begin{align*}
 \R(\A;S) &= \M(\A;S) - \min_{h\in \H} \M(h;S) \\
           &= \sum_{t=1}^T \ell^\A_t - \mathtt{opt_\H} \\
           &= \sum_{t=1}^T \ell^\mathcal{B}_t\cdot p - \mathtt{opt_\H} \\
           &= p\sum_{t=1}^T \ell^\mathcal{B}_t - \mathtt{opt_\H} \\
           & \leq p(\mathtt{opt_\H} + c\cdot \bigl(\sqrt{{T\ln T}}\bigr)) - \mathtt{opt_\H} \\
           &= (p-1)\mathtt{opt_\H} + p\cdot c\bigl(\sqrt{{T\ln T}}\bigr) \\
           &\leq (p-1)\mathtt{opt_\H} + c\bigl(\sqrt{{T\ln T}}\bigr).
\end{align*}
\end{proof}

\begin{proof}[Proof of Item~\ref{item:agn2} (Nonnegative Regret)]
Assume that $\A$ is a list-learning rule with $\k_\A = \k_\H$.
We will show that $\A$ has a non-negative regret, by showing that for a sufficient small $p>0$, there is an input sequence $S$ with arbitrarily large length $T$, such that $\mathtt{opt_\H} \leq \lfloor p\cdot T\rfloor$ but the expected number of mistakes made by $\A$ is at least $\lfloor p\cdot T\rfloor$. 
We will show that this holds for any $p\leq \frac{1}{\k_\H+1}$.
To simplify presentation we assume here that $p\cdot T\in \mathbb{N}$.

Let $T \in \mathbb{N}$. Since $\L_{\k_{\H}-1}(\H)=\infty$, there is a complete $\k_\H$-ary mistake tree $\T$ of length $T$ which is shattered by $\H$. 
We construct the hard sequence $S$ randomly as follows:
\begin{enumerate}
\item First draw a branch $B$ in $\T$ uniformly at random by starting at the root and at each step proceed on an outgoing edge which is chosen uniformly at random. 
\item Let $S_B$ denote the sequence of $T$ examples corresponding to the random branch, and let $L_t$ be the set of labels corresponding to the outgoing edges from the $t$'th node in the branch $B$. (In particular, notice that $\lvert L_t\rvert = \k_\H$ and $y_t\in L_t$.)
\item Pick uniformly a random subsequence of $p\cdot  T$ examples from $S_B$, and for each of the examples $(x_t,y_t)$ in the subsequence, replace the label $y_t$ with a random label $y_t'$, chosen uniformly such that $y_t'\notin L_t$. 
\item Denote by $S$ the resulting random sequence.
\end{enumerate}

Notice that $\E[\mathtt{opt_\H}] \leq p\cdot T$; indeed, $S_B$ is realizable and $S$ differs from it in $p\cdot T$ places.

It remains to show that the expected number of mistakes made by $\A$ is $\geq p\cdot T$. 
By linearity of expectation, it suffices to show that the probability that $\A$ makes a mistake at each step is $\geq p$.
For each $y \in \Y$ and $t\leq T$, let $p_t(y)$ be the probability that the label $y$ is in the predicted list of $\A$ at time-step $t$ and let $q_t(y)=1-p_t(y)$ be the probability that $y$ is not in the predicted list of $\A$ in time $t$.
Let $\alpha=\sum_{y\in L_t}p_t(y)$ and $\beta=\sum_{y\in {\Y \setminus L_t}}p_t(y)$;
thus, $\alpha + \beta = \lvert L_t\rvert = \k_\H$.
Then,
the probability that $\A$ makes a mistake at the time-step $t$ is:
\begin{align*}
   l_t(\A) 
   &= (1-p)\cdot \sum_{y\in L_t}q_t(y)\cdot \frac{1}{\lvert L_t\rvert } + p\cdot\sum_{y\notin L_t}q_t(y)\cdot\frac{1}{\lvert\Y\setminus L_t\rvert} \\
&=(1-p)\cdot\frac{\k_\H - \alpha}{\k_\H} + p\cdot\frac{\lvert\Y\rvert - \k_\H - \beta}{\lvert\Y\rvert - \k_\H} \\
&=(1-p)\cdot\frac{\beta}{\k_\H} + p\cdot\frac{\lvert\Y\rvert - \k_\H - \beta}{\lvert\Y\rvert - \k_\H}. 
\end{align*}

It suffices to show that $l_t(\A)\geq p$ for every possible value of $\beta\in [0,\k_\H]$. Indeed:
\begin{itemize}
    \item If $\beta=0$ then $l_t(\A)=p$. 
    \item If $\beta\geq 1$, then $l_t(\A) \geq (1-p)\frac{\beta}{\k_\H}\geq (1-p)\frac{1}{k_{\H}}\geq p$, because $p\leq \frac{1}{\k_\H+1}$.
    \item If $0<\beta<1$, then since $l_t(\A)$ is a linear function in $\beta$, getting values $\geq p$ for $\beta\in\{0,1\}$ then also for the middle range, $l_t(\A)\geq p$.
\end{itemize}

To conclude this analysis, for each value of $p$ smaller than $1/(\k_\H+1)$, each list-learning rule~$\A$, with list size $\k_\H$, achieves a non-negative regret for some input sequence with arbitrarily large length $T$.

\end{proof}

\begin{proof}[Proof of Item~\ref{item:agn3} (Unbounded Regret)]
Let $\A$ be a list learning rule with $\k_\A = \k_\H$, and let $h',h''\in \H$ and $x \in \X$ such that $h'(x),h''(x)$ are distinct lists of size $\k_\H$. 
Produce a random sequence of length $T$ whose examples are of the form $(x,y_t)$, where $y_t$ is a random label drawn uniformly from $h'(x)\cup h''(x)$.
Notice that the expected number of mistakes of $\A$ is at least $T\cdot\bigl(1 - \frac{\k_\H}{U}\bigr)$, where $U= \lvert h'(x)\cup h''(x)\rvert > \k_\H$.
Thus it remains to show that $\mathtt{opt_\H}$ does better.
Consider the random variable $\Delta$ counting how many of the $y_t$'s belong to the symmetric difference $h'(x)\Delta h''(x)$. 
If $y_t$ is in the intersection of $h'(x)$ and $h''(x)$ then both of the functions do not make a mistake on $(x,y_t)$, and therefore such examples do not contribute to $\mathtt{opt_\H}$. 
Otherwise, if $y_t$ is in the symmetric difference $h'(x)\Delta h''(x)$ then the probability that $y_t$ belongs to $h'(x)$ is equal to the probability it belongs to $h''(x)$, and both are equal to~$1/2$. 
Thus, conditioned on $\Delta$, the number of mistakes made by $h'$ or $h''$ behaves like a binomial random variables $X', X''$ with parameters $n=\Delta, p=1/2$, where $X'+X''=\Delta$ and $\E[\mathtt{opt_\H} \vert \Delta]$ is at most $\E[\min\{X', X''\}\vert \Delta]=\E[\min\{X', \Delta-X'\}\vert \Delta]$:
\[\E[\mathtt{opt_\H} \vert \Delta] 
\leq \E[\min\{X', \Delta-X'\}\vert \Delta]  
= \frac{\Delta}{2} - \E\Big[\Big| X'-\frac{\Delta}{2}\Big| \vert \Delta\Big].\]
Now, notice that
\[\E\Big[\Big| X'-\frac{\Delta}{2}\Big| \vert \Delta\Big] = \sqrt{\mathsf{Var}(X')} = \frac{\sqrt{\Delta}}{2},\]
and by linearity of expectation,
\[\E[\mathtt{opt_\H}] = \E[\E[\mathtt{opt_\H}\vert \Delta]] \leq \frac{1}{2}\E[\Delta] - \frac{1}{2}\E[\sqrt{\Delta}] \leq  \E[L_T(\A)] -  \frac{1}{2}\E[\sqrt{\Delta}],\]
where the last inequality holds because 
\[\E[\Delta]=T\cdot \frac{U-\lvert h'(x)\cap h''(x)\rvert}{U}=T\cdot\frac{U-(2\k_\H - U)}{U}=2T\cdot\frac{U - \k_\H}{U}\leq 2\E[L_T(\A)].\]
It thus remains to show that $\E[\sqrt{\Delta}] = \Omega(\sqrt{T})$. Indeed, since $\Delta$ is a binomial random variable with horizon $T$ and probability $p=2\cdot \frac{U - \k_\H}{U}$,
\[\E[\sqrt{\Delta}] 
\geq \sqrt{T\cdot p}-\frac{1-p}{2\sqrt{T\cdot p}} 
   = \Omega(\sqrt{T}).\]
The latter follows because $\sqrt{x} \geq 1+\frac{x-1}{2}-\frac{(x-1)^2}{2}$ and therefore for any nonnegative random variable $X$, $\E[\sqrt{X}] \geq \sqrt{\E[X]}\left(1-\frac{\mathsf{Var}(X)}{2\E[X]^2}\right)$.
Further, this bound can be expressed in terms of $\k_\H$, by the observation that 
\begin{equation}
   p=2\cdot \frac{U - \k_\H}{U} \geq \frac{2}{\k_\H+1}, \tag{$\k_\H + 1\leq U\leq 2\cdot \k_\H$} 
\end{equation}
and thus,
\[\E[\sqrt{\Delta}] 
\geq \sqrt{T\cdot p}-\frac{1-p}{2\sqrt{T\cdot p}} 
\geq \sqrt{T\cdot\frac{2}{\k_\H+1}}-\frac{1-\frac{2}{\k_\H+1}}{2\sqrt{T\cdot \frac{2}{\k_\H+1}}}
   = \Omega(\sqrt{T}).\]

\end{proof}

\end{proof}

\section{Future Research}\label{sec:future}
There are plenty of directions for keep exploring list online learning, here we consider a few natural ones that arise in this work.
\begin{enumerate}
\item Our characterization of the agnostic case only applies when the label space is finite. (Quantitatively, our upper bound on the regret depends on the size of the label space.) Does the characterization extends to infinite label space? More generally, it will be interesting to derive tighter lower and upper bounds on the optimal regret.
\item Study specific list learning tasks. For example, for linear classification with margin, we demonstrated an upper bound using the list Perceptron algorithm.  It will be interesting to determine whether this algorithm achieves an optimal mistake bound, or perhaps there are better algorithms?
\item What is the optimal expected mistake bound for randomized list learners? Is there a natural description of an optimal randomized list learner?
\item We study list online learning in the context of multiclass online prediction using the 0/1 classification loss. Extensions to weighted and continuous losses are left for future research.
\item Study the expressivity of pattern classes (see Section~\ref{app:patternexample}).
\end{enumerate}


\section*{Acknowledgements}
We thank Amir Yehudayoff for insightful discussions.


\bibliographystyle{abbrvnat}
\bibliography{sample}

\begin{thebibliography}{18}
\providecommand{\natexlab}[1]{#1}
\providecommand{\url}[1]{\texttt{#1}}
\expandafter\ifx\csname urlstyle\endcsname\relax
  \providecommand{\doi}[1]{doi: #1}\else
  \providecommand{\doi}{doi: \begingroup \urlstyle{rm}\Url}\fi

\bibitem[Alon et~al.(2021)Alon, Hanneke, Holzman, and Moran]{AlonHHM21}
N.~Alon, S.~Hanneke, R.~Holzman, and S.~Moran.
\newblock A theory of {PAC} learnability of partial concept classes.
\newblock In \emph{62nd {IEEE} Annual Symposium on Foundations of Computer
  Science, {FOCS} 2021, Denver, CO, USA, February 7-10, 2022}, pages 658--671.
  {IEEE}, 2021.
\newblock \doi{10.1109/FOCS52979.2021.00070}.
\newblock URL \url{https://doi.org/10.1109/FOCS52979.2021.00070}.

\bibitem[Angelopoulos and Bates(2021)]{AB21}
A.~N. Angelopoulos and S.~Bates.
\newblock A gentle introduction to conformal prediction and distribution-free
  uncertainty quantification.
\newblock \emph{CoRR}, abs/2107.07511, 2021.
\newblock URL \url{https://arxiv.org/abs/2107.07511}.

\bibitem[Auer and Long(1999)]{auer1999structural}
P.~Auer and P.~M. Long.
\newblock Structural results about on-line learning models with and without
  queries.
\newblock \emph{Machine Learning}, 36\penalty0 (3):\penalty0 147--181, 1999.

\bibitem[Ben-David et~al.(2009)Ben-David, P{\'a}l, and
  Shalev-Shwartz]{Ben-david09agnostic}
S.~Ben-David, D.~P{\'a}l, and S.~Shalev-Shwartz.
\newblock Agnostic online learning.
\newblock Dec. 2009.
\newblock 22nd Conference on Learning Theory, COLT 2009 ; Conference date:
  18-06-2009 Through 21-06-2009.

\bibitem[Beygelzimer et~al.(2019)Beygelzimer, Pál, Szörényi,
  Thiruvenkatachari, Wei, and Zhang]{BeygelzimerPSTW19}
A.~Beygelzimer, D.~Pál, B.~Szörényi, D.~Thiruvenkatachari, C.-Y. Wei, and
  C.~Zhang.
\newblock Bandit multiclass linear classification: Efficient algorithms for the
  separable case.
\newblock In K.~Chaudhuri and R.~Salakhutdinov, editors, \emph{ICML}, volume~97
  of \emph{Proceedings of Machine Learning Research}, pages 624--633. PMLR,
  2019.
\newblock URL
  \url{http://dblp.uni-trier.de/db/conf/icml/icml2019.html#BeygelzimerPSTW19}.

\bibitem[Brukhim et~al.(2022)Brukhim, Carmon, Dinur, Moran, and
  Yehudayoff]{Brukhim22Multiclass}
N.~Brukhim, D.~Carmon, I.~Dinur, S.~Moran, and A.~Yehudayoff.
\newblock A characterization of multiclass learnability.
\newblock \emph{CoRR}, abs/2203.01550, 2022.
\newblock \doi{10.48550/arXiv.2203.01550}.
\newblock URL \url{https://doi.org/10.48550/arXiv.2203.01550}.

\bibitem[Cesa-Bianchi and Lugosi(2006)]{cesa2006prediction}
N.~Cesa-Bianchi and G.~Lugosi.
\newblock \emph{Prediction, learning, and games}.
\newblock Cambridge university press, 2006.

\bibitem[Charikar and Pabbaraju(2022)]{CharikarP22}
M.~Charikar and C.~Pabbaraju.
\newblock A characterization of list learnability.
\newblock \emph{CoRR}, abs/2211.04956, 2022.
\newblock \doi{10.48550/arXiv.2211.04956}.
\newblock URL \url{https://doi.org/10.48550/arXiv.2211.04956}.

\bibitem[Crammer and Singer(2003)]{CramerS03}
K.~Crammer and Y.~Singer.
\newblock Ultraconservative online algorithms for multiclass problems.
\newblock \emph{J. Mach. Learn. Res.}, 3\penalty0 (null):\penalty0 951–991,
  mar 2003.
\newblock ISSN 1532-4435.
\newblock \doi{10.1162/jmlr.2003.3.4-5.951}.
\newblock URL \url{https://doi.org/10.1162/jmlr.2003.3.4-5.951}.

\bibitem[Daniely et~al.(2015)Daniely, Sabato, Ben-David, and
  Shalev-Shwartz]{Daniely15Multiclass}
A.~Daniely, S.~Sabato, S.~Ben-David, and S.~Shalev-Shwartz.
\newblock Multiclass learnability and the erm principle.
\newblock \emph{J. Mach. Learn. Res.}, 16\penalty0 (1):\penalty0 2377–2404,
  jan 2015.
\newblock ISSN 1532-4435.

\bibitem[Duda and Hart(1973)]{DudaHart1973}
R.~O. Duda and P.~E. Hart.
\newblock \emph{Pattern Classification and Scene Analysis}.
\newblock John Willey \& Sons, New Yotk, 1973.

\bibitem[Hazan(2019)]{hazan2019introduction}
E.~Hazan.
\newblock Introduction to online convex optimization.
\newblock \emph{arXiv preprint arXiv:1909.05207}, 2019.

\bibitem[Littlestone(1988)]{Littlestone88online}
N.~Littlestone.
\newblock Learning quickly when irrelevant attributes abound: A new
  linear-threshold algorithm.
\newblock In \emph{Machine Learning}, pages 285--318, 1988.

\bibitem[Long(2001)]{LongPartial}
P.~Long.
\newblock On agnostic learning with $\{0, *, 1\}$-valued and real-valued
  hypotheses.
\newblock In \emph{Proceedings of the 14th Annual Conference on Learning Theory
  and 5th European Conference on Computational Learning Theory}, 2001.

\bibitem[Rosenblatt(1958)]{rosenblatt1958perceptron}
F.~Rosenblatt.
\newblock {The perceptron: A probabilistic model for information storage and
  organization in the brain.}
\newblock \emph{Psychological Review}, 65\penalty0 (6):\penalty0 386--408,
  1958.
\newblock ISSN 0033-295X.
\newblock \doi{10.1037/h0042519}.
\newblock URL \url{http://dx.doi.org/10.1037/h0042519}.

\bibitem[Shafer and Vovk(2008)]{ShaferV08}
G.~Shafer and V.~Vovk.
\newblock A tutorial on conformal prediction.
\newblock \emph{J. Mach. Learn. Res.}, 9:\penalty0 371--421, 2008.
\newblock \doi{10.5555/1390681.1390693}.
\newblock URL \url{https://dl.acm.org/doi/10.5555/1390681.1390693}.

\bibitem[Shalev{-}Shwartz(2012)]{Shalev-Shwartz12survey}
S.~Shalev{-}Shwartz.
\newblock Online learning and online convex optimization.
\newblock \emph{Found. Trends Mach. Learn.}, 4\penalty0 (2):\penalty0 107--194,
  2012.
\newblock \doi{10.1561/2200000018}.
\newblock URL \url{https://doi.org/10.1561/2200000018}.

\bibitem[Vovk et~al.(2005)Vovk, Gammerman, and Shafer]{VGS05}
V.~Vovk, A.~Gammerman, and G.~Shafer.
\newblock \emph{Algorithmic Learning in a Random World}.
\newblock Springer-Verlag, Berlin, Heidelberg, 2005.
\newblock ISBN 0387001522.

\end{thebibliography}

\appendix

\section{Pattern Classes vs.\ Hypothesis Classes}\label{app:patternexample}

In this section we compare between hypothesis classes and pattern classes in terms of their expressivity to model learning tasks.

Recall that a pattern class $\P$ is a collection of finite sequences of examples, which is downward closed: for every $S\in \P$, if $S'\subseteq S$ is a subsequence of $S$ then $S'\in \P$,
and also that the induced pattern class $\P(\H)$ of an hypothesis class $\H$ is:
\[\P(\H)=\{S\in\Z^\star : S \text{ is consistent with some } h\in \H\}.\]

There are two basic properties which seem to capture the gap between pattern classes and hypothesis classes. 
These two properties hold in any induced pattern class $\P(\H)$, but (as shown below)  they do not necessarily hold in a general pattern class $\P$. 

\paragraph{Contradiction-free.} 
A pattern class $\P$ is called contradiction-free if every pattern in it has no contradictions, 
namely it does not contain the same example twice with different labels. 
Clearly, $\P(\H)$ is contradiction-free for every hypothesis class $\H$.
Thus, any pattern class $\P$ which contains a contradiction cannot be extended by an hypothesis class $\H$.
Below we focus only on contradiction-free pattern classes.

\paragraph{Symmetry.} 
A pattern class $\P$ is called symmetric if it is closed under taking permutations: if $S\in \P$ then $\pi(S)\in \P$ for every permutation $\pi$ of the elements in $S$.
Clearly, $\P(\H)$ is symmetric for every hypothesis class $\H$.
Non-symmetric pattern classes can be naturally used to express adaptive/dynamic hypotheses (or hypotheses with memory),
which are functions of the form $h:\X^\star\to\Y$ mapping a finite sequence $x_1,\ldots,x_t$ to an output $y_t\in \Y$.
($x_t$ is thought of as the current point, $y_t$ is its label, and $x_1,\ldots,x_{t-1}$ are the previous points.)
Thus, non-symmetric pattern classes allows us to study the learnability of adaptive hypothesis classes.

This suggests the following question: \emph{Can learnable pattern classes which are contradiction-free (but possibly non-symmetric) be extended by a learnable hypothesis class?}
I.e.\ is it the case that for every contradiction-free pattern class $\P$ with $\L(\P)<\infty$ there exists an hypothesis class $\H$
such that $\P(\H)\supseteq \P$ and $\L(\H)<\infty$?

The answer to this question is no, as the following example shows.
Let $\X$ be an infinite set and $\Y=\{0,1\}$. Consider the pattern class $\P\subset\left(\X\times\Y\right)^*$ containing all patterns $\{\left(x_i,y_i\right)\}_{i=1}^n$ such that their restriction to $\Y$, $\{y_i\}_{i=1}^n$, is non-increasing: if $i<j$ then $y_i\geq y_j$, and such that $y_i\neq y_j \Rightarrow x_i\neq x_j$ (i.e. contradiction-free property).

Not hard to see that $\P$ is online learnable by the learner that predicts the label $1$, until she makes a mistake (namely, the true label is $0$), and then predicts the label $0$ for the rest of the input sequence of examples.

Now assume that $\H\subset\Y^\X$ is an hypothesis class, such that $\P\subseteq \P(\H)$. We will show that $\H$ cannot be online learnable, by showing that $\L(\H)\geq d$ for all $d \in \mathbb{N}$. Let $d \in \mathbb{N}$, we will show that there is a complete binary mistake tree of depth $d$ which is shattered by $\H$. In fact, we will show a stronger result; we will show that almost $\it{any}$ complete binary mistake tree of depth~$d$ is shattered by $\H$, with the only restriction of different values of $x\in \X$ in each root-to-leaf path in the tree, to avoid contradictions. 
Let $\T$ be such a tree, and let $S$ be the induced sequence of some path in $\T$. 
We claim that $S$ is realized by $\H$. Indeed, reorder the elements of $S$ to get a sequence $S'$ with non-increasing order of the labels. I.e., permute $S$ to get a sequence $S'\in \P\subseteq \P(\H)$. Since $\P(\H)$ is symmetric it follows that $S\in \P(\H)$ and hence $S$ is realized by $\H$. Thus, $\T$ is shattered by $\H$ as claimed.

What about learnable pattern classes $\P$ which are both free of contradictions and symmetric? is there always a learnable hypothesis class $\H$ such that $\P \subset \P(\H)$? This remains an interesting open question for future research. (We remark that an equivalent open question was asked by~\citet{AlonHHM21}.)





\section{Tightness of Bounds for Randomized Learners}\label{app:randexamples}
Here we give two examples for pattern classes, one class $\P_1$ with $\R_k(\P_1)=\frac{1}{k+1}\L_k(\P_1)$ and second class $\P_2$ with $\R_k(\P_2)=\L_k(\P_2)$.

\paragraph{Lower bound is tight.}
The first class is the induced pattern class $\P_1=\P(\H)$ of the hypothesis class of all functions from $[d]$ to $[k+1]$:
\[\H=[k+1]^{[d]}.\]
Notice that $\L_k(\P_1)=d$.

Now, we will define a randomized algorithm $\A$ which makes at most $\frac{1}{k+1}\L_k(\P_1)=\frac{d}{k+1}$ mistakes in expectation on every realizable sequence. Notice that in a realizable sequence $S=\{(x_i,y_i)\}_{i=1} ^{T}$ by $\P_1$, if $x_i=x_j$ then $y_i=y_j$. Define the algorithm $\A$ so that for each input example $x_i$, the prediction will be a list of $k$ labels, chosen uniformly at random. If the algorithm already seen the example $x_i$, then the predicted set will consists of the matching known label and completed to a list of size $k$ arbitrarily. for each $x\in[d]$, the algorithm $\A$ may make a mistake on $x$ in at most one single time step. The probability that the learner will make a mistake on $x$ at this time step is $\frac{1}{k+1}$. Hence, the expected number of mistakes of $\A$ on $S$ is at most $\frac{d}{k+1}$.

\paragraph{Upper bound is tight.}
The second class is the induced pattern class $\P_2=\P(\H)$ of the following hypothesis class $\H\subset\{0,1,\ldots,k\}^{\mathbb{N}}$:
\[\H=\{h\in \{0,1,\ldots,k\}^{\mathbb{N}} : \lvert h^{-1}(0)\rvert\leq d\}.\]
Notice that $\P_2$ is online learnable by the learner that always predicts the list $\{1,2,\ldots,k\}$, and that $\L_k(\P_2)=d$.
Indeed, 
take a complete $(k+1)$-ary mistake tree of depth $d$, where all nodes with same depth are associated with the same $x\in \mathbb{N}$ and each two nodes with different depths are associated with different items from $\mathbb{N}$. This tree is shattered by $\P_2$, hence $\L_k(\P_2)\geq d$. In addition, there is no complete shattered tree with depth larger than $d$, since a path with all zeros labels must have size of at most $d$, hence $\L_k(\P_2)\leq d$. 

Now let $\A$ be a randomized learner. We will show that for each $\epsilon>0$, there is a realizable sequence $S=\{(x_t,y_t)\}_{t=1}^T$ of length $T(\epsilon)$, such that $\A$ makes at least $d-\epsilon$ mistakes on $S$ in expectation.
First, we choose the $x_t$'s such that they are all distinct, e.g. $x_t=t$.
The labels $y_t$ are defined as follows:
Let $\epsilon>0$. Let $P_t$ be the (maybe random) predicted set of $\A$ at time $t$. If $\Pr(0\in P_t)\leq \frac{\epsilon}{d}$, then we will set $y_t=0$. If we already set $y_t=0$ $d$ times, then we will choose a different value for $y_t$, uniformly at random from $\{1,\ldots,k\}$, to keep the realizability property of the input sequence.
If $\Pr(0\in P_t)> \frac{\epsilon}{d}$, then we will choose the label $y_t$ uniformly at random from the set $\{1,\ldots,k\}$.
By this construction of the input sequence, we get the following result; in the case that the input sequence contains $d$ entries with $y_t=0$, it means that $\A$ makes at least $d\left(1-\frac{\epsilon}{d}\right)=d-\epsilon$ mistakes in expectation.
In the case that input sequence contains less than $d$ entries with $y_t=0$, then for input sequence with length $T(\epsilon)\geq\frac{kd^2}{\epsilon}+d-1$, $\A$ will make at least $\frac{kd^2}{\epsilon}\left(1-\frac{1}{k}\left(k-\frac{\epsilon}{d}\right)\right)=d$ mistakes in expectation.

\end{document}